\documentclass[11pt, a4paper, google]{google}

\usepackage[authoryear, sort&compress, round]{natbib}
\usepackage[T1]{fontenc}
\usepackage[utf8]{inputenc}
\usepackage{titletoc}

\usepackage{amsmath,amssymb,amsthm}
\DeclareFontFamily{U}{msb}{}
\DeclareFontShape{U}{msb}{m}{n}{<-> msbm10}{}
\usepackage{booktabs,tabularx,array,multirow}
\usepackage{graphicx,xcolor}
\usepackage{wrapfig}
\usepackage{tikz}
\usepackage{placeins}
\usepackage{xspace}
\usepackage{algorithm,algpseudocode}
\usepackage{listings}
\usepackage{hyperref,xurl}
\usepackage{colortbl}
\usepackage{comment}
\usepackage{adjustbox}

\definecolor{GoogleBlue}{HTML}{4285F4}
\definecolor{GoogleRed}{HTML}{EA4335}
\definecolor{GoogleYellow}{HTML}{FBBC05}
\definecolor{GoogleGreen}{HTML}{34A853}
\definecolor{Ink}{HTML}{243247}
\providecommand{\resultcell}[2]{#1{\scriptstyle\pm #2}}
\providecommand{\bestscore}[1]{\mathbf{#1}}
\newsavebox{\ARCselectiontablebox}

\newcommand{\KL}{\operatorname{KL}}

\newcommand{\sg}{\operatorname{sg}}

\newcommand{\noadv}{\texttt{<NO\_ADVICE>}}

\newcommand{\codepath}[1]{\nolinkurl{#1}}

\newcolumntype{Y}{>{\raggedright\arraybackslash}X}

\newtheorem{arcassumption}{Assumption}
\newtheorem{arctheorem}{Theorem}

\newtheorem{arclemma}{Lemma}
\newtheorem{arccorollary}{Corollary}
\newtheorem{arcrestatedlemma}{Lemma}

\theoremstyle{definition}

\newcommand{\arcE}{\mathbb E}

\newsavebox{\ARCpanelbody}
\newenvironment{ARCPanel}[2]{%
  \par\addvspace{9pt}\begingroup
  \def\ARCpanelcolor{#1}\def\ARCpaneltitle{#2}%
  \def\ARCpaneltext{white}\def\ARCgold{GoogleYellow}%
  \ifx\ARCpanelcolor\ARCgold\def\ARCpaneltext{Ink}\fi
  \begin{lrbox}{\ARCpanelbody}%
  \begin{minipage}{\dimexpr\linewidth-17pt\relax}%
  \small\raggedright\setlength{\parindent}{0pt}\setlength{\parskip}{3pt}%
}{%
  \end{minipage}\end{lrbox}%
  \noindent\begin{tikzpicture}
    \node[draw=\ARCpanelcolor,fill=\ARCpanelcolor!5,
      line width=.7pt,rounded corners=3pt,inner sep=8pt,
      outer sep=0pt,anchor=north west] (body) at (0,0)
      {\usebox{\ARCpanelbody}};
    \path[fill=\ARCpanelcolor,rounded corners=3pt]
      ([yshift=18pt]body.north west) rectangle
      ([yshift=-2pt]body.north east);
    \node[anchor=west,inner sep=0pt,text=\ARCpaneltext,
      font=\rmfamily\bfseries\small] at
      ([xshift=8pt,yshift=8pt]body.north west) {\ARCpaneltitle};
  \end{tikzpicture}\par\endgroup
}
\lstdefinestyle{ARCprompt}{%
  basicstyle=\ttfamily\small,frame=none,backgroundcolor={},
  aboveskip=0pt,belowskip=0pt,xleftmargin=0pt,xrightmargin=0pt,
  breaklines=true,breakautoindent=false,columns=fullflexible,
  keepspaces=true,showstringspaces=false}
\lstdefinestyle{ARCtrace}{%
  basicstyle=\ttfamily\small,frame=none,backgroundcolor={},
  aboveskip=1pt,belowskip=1pt,xleftmargin=0pt,xrightmargin=0pt,
  breaklines=true,breakautoindent=false,columns=fullflexible,
  keepspaces=true,showstringspaces=false}
\newcommand{\ARCevent}[2]{%
  \par\addvspace{4pt}\noindent
  {\color{#1!75!black}\rmfamily\bfseries\small[#2]}\par\nobreak}

\newcolumntype{C}{>{\centering\arraybackslash}X}
\makeatletter
\newcommand{\ARCcaptionof}[1]{\def\@captype{#1}\caption}
\makeatother

\uselogo{} 

\newcommand{\method}{\textsc{AdviSD}\xspace}

\title{\method: Learning to Advise Frontier LLMs via Targeted Multi-Turn Self-Distillation}

\correspondingauthor{rishabha@usc.edu, \{hejiec, shashali, shanchanwu, soarik\}@google.com \\ \textup{* This work was done while Rishabh was a Student Researcher at Google.}}

\renewcommand{\today}{}

\author[1, 2*]{Rishabh Agrawal}
\author[1]{Hejie Cui}
\author[1]{Shasha Li}
\author[1]{Shanchan Wu}
\author[1]{Sercan \"{O}. Ar{\i}k}

\affil[1]{\thepa{}{}}
\affil[2]{University of Southern California}

\begin{abstract}
A small trainable advisor can steer a frozen language-model executor using natural-language advice. In addition to learning from task rewards, the advisor can use feedback from completed interactions to improve its advice. However, a plausible correction need not change execution, yet learning from such corrections can still affect the advisor's future decisions in other contexts. In a shared-parameter model, we prove that such corrections can limit learning if their targets favor useful advice less strongly than those of other corrections. Keeping them less often than the rest improves the model's eventual performance compared to learning from every correction. Motivated by this, our method, Advisor Self-Distillation (\method), pairs outcome-based reinforcement learning with self-distillation from a feedback-conditioned copy of the advisor selectively. Reflection proposes corrections, and the advisor scores the same recorded executor response with and without its issued advice, using the magnitude of the difference to select decisions for supervision. This approach does not require executor likelihoods or additional executor rollouts. Experiments with Qwen3-8B advisors for Gemini and Claude show that \method outperforms advisor-GRPO by 4.2–6.4 percentage points on BFCL-v3 and by 3.9–5.1 score points on EnvScaler. The trained advisors generalize to out-of-domain tasks and transfer across different executor versions and model families. \method also beats matched-count random selection, supporting the value of its selection rule.
\end{abstract}

\begin{document}

\maketitle

\section{Introduction}
\label{sec:introduction}

Frontier language models are usually served through APIs that accept
queries but do not let users change the model weights. When such a
model acts as an agent, a smaller trainable advisor can adapt it from
the outside: the advisor observes the interaction and recommends what
the frozen agent, which we call the executor, should do next
\citep{li2023guiding,NEURIPS2025_3f055edc}. The advisor can be trained
with reinforcement learning on the rewards of completed interactions
\citep{asawa2026how}. However, an episode reward summarizes task
performance without specifying which of the advisor's decisions should
have been different, or how.

Completed interactions contain more specific evidence: executor
responses, tool results, and task checks. These observations and the
episode reward form the feedback that reflection uses to propose
revisions \citep{liu2026hero,yeo2026hint}. In feedback-conditioned
self-distillation, a teacher that sees this feedback supervises a
student that sees only the original context
\citep{hubotter2026reinforcement,agrawal2026reinforcement}. For an advisor, however,
the learned advice acts through another model: revising it need not
change what the executor does.

Consider an executor that searches reliably for suitable flights
without guidance but sometimes guesses the passenger identifier when
booking. After a reservation fails because that identifier is wrong,
reflection may propose looking up the passenger and using the returned
identifier in the booking. It may also propose more detailed advice for
the earlier flight search. Both revisions are valid, but their
usefulness differs for this executor: the booking revision addresses
the error, whereas the search revision elaborates a procedure the
executor already follows unaided. The teacher's supervision can
reflect both revisions, and learning from the redundant search
revision can still change the advisor's shared parameters and affect
advice elsewhere, for better or worse. Hence, we ask which
revisions provide useful supervision for a given advisor--executor
pair. Our contributions are as follows:

\textbf{A theoretical account of correction selection.}
Lemma~\ref{arc:lem:channel} separates agreeing with a feedback-conditioned
teacher from improving execution. We then study repeated learning in a
shared-parameter model where the teacher's preferences and the
executor's behavior under any given advice stay fixed during training.
As advice improves, preventable failures become less frequent, while
failures unaffected by advice persist. If corrections from these
persistent failures teach a weaker preference for useful advice, their
growing share of supervision limits learning. Retaining them less often
than other corrections raises the performance the advisor eventually
reaches, whether or not reward learning is added
(Theorems~\ref{arc:thm:limit}--\ref{arc:thm:combined}). By contrast,
randomly discarding corrections at the same rate across both types
leaves eventual performance unchanged when learning only from
corrections, but can improve it alongside reward learning
(Corollary~\ref{arc:cor:exposure}). Therefore, our matched-count random
control tests whether targeted selection helps beyond reducing
supervision (Section~\ref{sec:ablations}).

\begin{figure}[t]
\centering
\includegraphics[width=\linewidth]{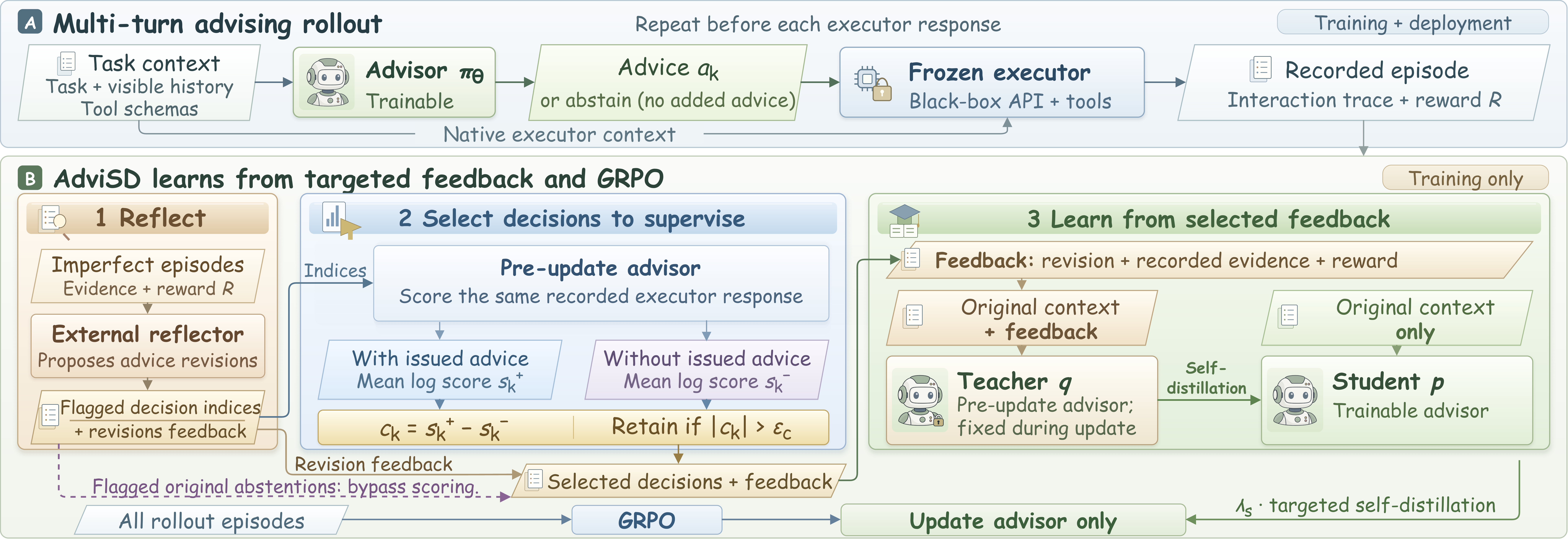}
\caption{\small
\textbf{Overview of \method.}
\textbf{(A) Multi-turn advising.}
Before each executor response, the advisor provides advice or abstains;
the frozen executor uses its native context and any issued advice.
\textbf{(B) Training from targeted feedback.}
Reflection proposes corrections at advice decisions in imperfect episodes.
Among flagged decisions, \method retains original abstentions and those
whose scores with and without issued advice differ by more than a
calibrated threshold. The pre-update advisor computes both on the same
recorded executor response. At retained decisions, a feedback-conditioned
copy of the pre-update advisor teaches the trainable advisor, which sees
only the original context. This self-distillation supplements GRPO
on all episodes.
}
\label{fig:overview}
\end{figure}

\textbf{\method: multi-turn advising with targeted feedback.}
\method separates proposing revisions from choosing where to learn
(Figure~\ref{fig:overview}). The advisor scores the same recorded
executor response under two contexts, one containing its issued advice
and one without it, so selection requires neither executor likelihoods
nor additional executor rollouts. We use the magnitude of the score difference as a predictive signal for selecting
decisions to supervise. At selected decisions, a feedback-conditioned
copy of the pre-update advisor supervises the trainable advisor.
This targeted self-distillation complements outcome-based GRPO
\citep{shao2024deepseekmath}. Reflection and scoring are used only
during training. At deployment, the advisor decides before each
executor turn whether to advise or abstain.

\textbf{Empirical results.}
With Qwen3-8B advisors for Gemini and Claude, \method has the highest
in-domain aggregates on BFCL-v3 and EnvScaler among the compared
methods. It exceeds advisor-GRPO by 4.2--6.4 percentage points on
BFCL-v3 and 3.9--5.1 score points on EnvScaler. Across these
settings, its 2.5--4.9-point advantage over matched-count random
selection supports choosing which decisions to supervise beyond
reducing supervision. Without retraining, \method improves
out-of-domain macro-averages by 2.7--3.6 points over standalone
execution. It also transfers across executor versions and model
families and exceeds transferred GRPO by 3.1 percentage points in
both cross-family directions (Section~\ref{sec:experiments}).

\section{Related Work}
\label{sec:related}

\textbf{Advising and prompt optimization.}
Frozen models can be adapted through reusable instructions or
context-dependent guidance. GEPA uses reflection to optimize reusable
instructions \citep{agrawal2026gepa}, while Directional Stimulus
Prompting, Matryoshka Pilot, and Advisor Models train smaller models
to guide frozen ones
\citep{li2023guiding,NEURIPS2025_3f055edc,asawa2026how}.
Self-Refine and Reflexion use feedback to revise outputs or guide
later attempts without updating model weights
\citep{madaan2023selfrefine,shinn2023reflexion}.
Our advisor-GRPO baseline follows Advisor Models' outcome-based
training through a response-level tool-use interface; \method adds
targeted feedback-conditioned self-distillation.

\textbf{Feedback-conditioned distillation.}
Feedback-conditioned on-policy self-distillation uses additional information
to supervise a policy on its own trajectories. SDPO obtains this
information from environment feedback or successful rollouts
\citep{hubotter2026reinforcement}, and DistIL optimizes forward cross-entropy
with sequence-level credit assignment \citep{agrawal2026reinforcement}.
Other approaches focus on constructing and allocating supervision.
HERO constructs turn-level feedback, and HinT-SD selects
failure-relevant action spans for self-distillation
\citep{liu2026hero,yeo2026hint}. LOPD builds teacher context from
retrieved experience, and DART-SD retrieves references for recovery
\citep{zhang2026latent,xu2026dart}. SAGE-OPD selects and weights
teacher supervision at individual turns \citep{zhou2026sage}.
In contrast, \method addresses which corrections to learn when the trained policy
advises a separate executor rather than directly performing the task.

\textbf{Predictive contrasts and selection.}
Comparing predictions made with different information can yield a
learning signal. RLCSD contrasts correct and incorrect hints, and OCSD
compares full and observation-ablated contexts
\citep{pan2026rlcsd,yang2026agentic}. PBSD and RLSD use paired predictions
to refine turn-level credit or token updates
\citep{tian2026pbsd,yang2026self}. \method applies paired scoring to a
separate executor's recorded response, with and without the issued
advice. It uses the contrast magnitude to select auxiliary supervision
for an advisor whose advice acts through a frozen executor, while
leaving the rollout batch's GRPO advantages unchanged.
Appendix~\ref{app:related} provides the full discussion and further
comparisons.
\section{Background and Problem Formulation}
\label{sec:setup}

\textbf{Advising a frozen executor.} Before each executor response $k$, the advisor reads a context $g_k$
(the visible interaction, tool schemas, and its earlier advice) and
samples an action $a_k\sim\pi_\theta(\cdot\mid g_k)$, which is either
advice text or the abstention sequence \noadv{}. The frozen executor
responds with $m_k\sim\rho(\cdot\mid h_k,e(a_k))$, where $h_k$ is its
native history and $e(a_k)$ is the advice text, or an empty string if
the advisor abstains. Advice goes into a temporary request rather than
the executor's persistent history. The advisor makes a new decision
before every response, including text-only responses and those that
follow tool results; parallel tool calls within one response share a
single decision. An episode ends with reward $R\in[0,1]$.

\textbf{Two learning signals.}
For $n$ rollouts of a task, GRPO assigns episode $i$ the advantage
$\widehat A_i=(R_i-\bar R)/(s_R+\epsilon_A)$, where $\bar R$ and $s_R$
are the group's reward mean and standard deviation, and
$\epsilon_A>0$ stabilizes the denominator. This advantage applies to
every advisor token generated in the episode, including abstentions.
We denote the clipped GRPO loss with reference-policy regularization
by $\mathcal L_{\rm base}$ \citep{shao2024deepseekmath}.

Feedback-conditioned self-distillation supervises individual advice
decisions. The \emph{teacher}, a copy of the pre-update advisor with
parameters $\bar\theta$, sees the original context augmented with
feedback from the completed interaction. The trainable \emph{student}
sees only the original context and learns to match the teacher's
next-token distributions, which are held fixed during optimization.
\method chooses which decisions to supervise, including originally
abstaining decisions flagged by reflection, and leaves the GRPO
advantages unchanged.
\section{Why the Choice of Corrections Matters}
\label{arc:sec:theory}\label{sec:theory}

We examine why fitting a teacher need not improve execution
(Section~\ref{arc:sec:channel}), then show how the retained corrections
determine the learning limit in a shared-parameter model
(Section~\ref{arc:sec:limit}).

\subsection{A single update through the executor}
\label{arc:sec:channel}

Fix an interaction state and advice prefix $h$, and let $p_\theta$ and
$q_h$ be positive student and teacher distributions on a fixed finite
token set $S_h$, possibly the full vocabulary. The student is
differentiable near the pre-update parameters $\bar\theta$. Choosing
token $v$, completing the advice with the pre-update advisor, and
running the executor induces an execution law $K_{h,v}$ over responses
and outcomes, excluding advice text. For a bounded task score $W$,
define
\[
V_h(v)=\arcE_{Z\sim K_{h,v}}[W(Z)],\qquad
J_h(\theta)=\sum_{v\in S_h}p_\theta(v)V_h(v).
\]
These are the expected score after choosing $v$ and its average under
the student. Only next-token probabilities vary during differentiation;
the teacher, support, completion policy, and execution laws remain
fixed.
Feedback does not directly provide each token's expected execution
value. We therefore compare $q_h$ with a normalized reference
$q_h^\alpha(v)\propto p_{\bar\theta}(v)e^{\alpha V_h(v)}$, $\alpha>0$,
which reweights the student toward higher-value tokens. This
\emph{value-tilted teacher} \citep{peters2010relative} is an analytical
benchmark; \method does not construct it.

\begin{arclemma}[Teacher fitting versus execution improvement]
\label{arc:lem:channel}
Under this setup, let
$\phi_h(v)=\nabla_\theta\log p_\theta(v)|_{\bar\theta}$ and
$g_h=\nabla_\theta\KL(p_\theta\|q_h)|_{\bar\theta}$.
For every fixed $\alpha>0$,
\begin{equation}
g_h=-\alpha\nabla J_h(\bar\theta)+e_h,\qquad
e_h=\arcE_{v\sim p_{\bar\theta}}
\left[\phi_h(v)\log\frac{q_h^\alpha(v)}{q_h(v)}\right].
\label{arc:eq:decomposition}
\end{equation}
\end{arclemma}

Reverse-KL descent at $\bar\theta$ toward the reference follows
$\alpha\nabla J_h$, whereas descent toward the actual teacher follows
$\alpha\nabla J_h-e_h$, whose residual can reinforce or oppose value
ascent. At an \emph{insensitive} prefix, every supported token induces
the same execution law, so $\nabla J_h=0$ and the reference equals
the student. Changing token probabilities cannot improve this local
objective, but teacher fitting can still change advice elsewhere
through shared parameters, for better or worse. Teacher agreement
alone does not tell us which. Appendix~\ref{arc:app:channel} gives
the full proof and extensions.

\subsection{Repeated updates and the learning limit}
\label{arc:sec:limit}
Lemma~\ref{arc:lem:channel} concerns one update. We now introduce a
simplified shared-parameter model to study how repeated learning
changes which failures occur and which corrections supply supervision.

\textbf{Setup.}
The advisor chooses between advice 1 and advice 2, selecting advice 1
with probability $p(\theta)=1/(1+e^{-\theta})$. A single log-odds
parameter $\theta$ is shared across a fixed mixture of \emph{sensitive}
($S$) and \emph{insensitive} ($I$) situations, each with positive
probability. In sensitive situations, advice 1 succeeds more often
than advice 2; in insensitive situations, both induce the same
execution law. These laws remain fixed, with success probabilities
in $(0,1)$, so expected success $J(\theta)$ increases with $\theta$.

A failure of type $j\in\{I,S\}$ supplies a fixed teacher $Q_j$,
positive on both advice choices, with target log-odds
$\ell_j=\log[Q_j(1)/Q_j(2)]$. Even when both advice choices lead to
identical executor behavior, the teacher may prefer one over the
other. We separately assume $\ell_I<\ell_S$, meaning that the
insensitive teacher assigns less probability to advice 1 than the
sensitive teacher. For example, $Q_I=(0.5,0.5)$ is neutral, while
$Q_S=(0.8,0.2)$ prefers advice 1. If the advisor already selects
advice 1 with probability $p>1/2$, learning from $Q_I$ pushes that
probability down toward $1/2$. Because $\theta$ is shared, this also
makes advice 1 less likely in sensitive situations, where it succeeds
more often. Thus, a neutral teacher can weaken useful advice elsewhere
without favoring advice 2.

Each episode contains a fixed positive number of independent,
identically distributed (situation, advice, outcome) samples from
this model. Failures supply correction proposals up to a fixed
positive cap, with uniform subsampling if the cap is exceeded.
Proposals of type $j$ are retained independently with fixed
probability $r_j\in(0,1]$; no gating means $r_I=r_S=1$.
The episode loss averages student-to-teacher reverse KL over retained
corrections and is zero if none remain. Samples and selections are
held fixed during differentiation.

\textbf{Retained supervision.}
Let $B(\theta)$ be the probability that an episode retains any
correction, and let $\omega_I(\theta)$ be the expected fraction of
insensitive corrections conditional on retaining at least one. The
mean target log-odds is
$\mu(\theta)=\omega_I(\theta)\ell_I+[1-\omega_I(\theta)]\ell_S$.
The quantity $B$ measures \emph{exposure}, how often episodes receive
supervision, while $\mu$ summarizes the \emph{composition} of that
supervision, the mixture of teacher targets. As $\theta$ increases,
sensitive failures become less frequent while the insensitive failure
rate stays fixed. The insensitive teacher therefore receives a growing
share of supervision, so $\mu$ decreases.

\begin{arctheorem}[The retained mixture sets the learning limit]
\label{arc:thm:limit}
Under this setup, with distillation alone, the expected gradient of
the sampled episode loss is
\begin{equation}
g(\theta)=B(\theta)\,p(1-p)\,[\theta-\mu(\theta)].
\label{arc:eq:gradient}
\end{equation}
The continuous-time update $\dot\theta=-g(\theta)$ converges from every
finite initialization to a unique equilibrium
$\theta^*=\mu(\theta^*)\in(\ell_I,\ell_S)$. Decreasing $r_I/r_S$
strictly increases both $\theta^*$ and $J(\theta^*)$. Changing the
episode size or proposal cap, or scaling both retention probabilities
by the same admissible positive factor, leaves the limit unchanged.
\end{arctheorem}

Each teacher contributes $p(1-p)(\theta-\ell_j)$ to the gradient;
averaging yields Eq.~\eqref{arc:eq:gradient}. Learning settles where
the advisor's log-odds equal the retained teachers' mean target.
Retaining insensitive corrections less often than sensitive ones
raises this target and the resulting learning limit. Independently
thinning both types at the same rate changes exposure but not the
target.
Appendices~\ref{arc:app:main-specialization}--\ref{arc:app:dynamics-general}
provide the proof, scope, and extensions.

This model retains corrections by situation type. \method instead uses
an observable predictive contrast (Section~\ref{sec:method}), whose
usefulness we evaluate through ablations (Section~\ref{sec:ablations}).
When reward learning is added, exposure can also affect eventual
performance (Section~\ref{sec:properties}), so the ablations include
a matched-count random control.
\section{\method: Advisor Self-Distillation}
\label{sec:method}

\method trains only the advisor, combining outcome-based GRPO with
targeted self-distillation (Figure~\ref{fig:overview}). An external
reflector proposes corrections, a predictive selector chooses decisions
to supervise, and a feedback-conditioned pre-update advisor teaches a
student that sees only the original context. These stages run only
during training (Algorithm~\ref{alg:arc-implementation};
Appendix~\ref{app:configuration}).

\subsection{Reflection proposes corrections}
\label{sec:reflection}

For each eligible imperfect episode $i$, the reflector uses executor
responses, tool outcomes, and checks to flag at most $b_{\rm refl}$
advice decisions $\mathcal J_i$ with correction feedback. Later events
can explain failures, but proposed advice uses information available
at the original decision
(Appendices~\ref{app:prompts} and~\ref{app:reward}).

\subsection{A paired score selects where to learn}
\label{sec:score}

\textbf{Scoring.}
Let $y_k=(y_{k,1},\ldots,y_{k,T_k})$ denote the recorded executor
response, serialized and tokenized with the advisor's tokenizer.
It includes tool calls in their recorded order but excludes subsequent
tool results. We construct two scoring contexts, $C_k^+$ and $C_k^-$,
from the request sent to the executor rather than from the advisor's
context $g_k$. Both contain the same pre-response history and tool
schemas and differ only in the \emph{issued} advice, which $C_k^+$
includes and $C_k^-$ omits. The pre-update advisor scores each token
of $y_k$ under both contexts:
\begin{equation}
c_k=\frac{1}{T_k}\sum_{t=1}^{T_k}
\left[\log\pi_{\bar\theta}(y_{k,t}\mid C_k^+,y_{k,<t})
-\log\pi_{\bar\theta}(y_{k,t}\mid C_k^-,y_{k,<t})\right].
\label{eq:score}
\end{equation}
The magnitude $|c_k|$ indicates how strongly the issued advice changes
the advisor's prediction of the recorded response. We use this
predictive signal to select whole advice decisions for supervision.
Both scores are computed by the pre-update advisor on the same
recorded response, so selection requires neither executor likelihoods
nor additional executor rollouts.

\textbf{Calibration and selection.}
\label{sec:calibration}\label{sec:selection} We calibrate the gate on prediction changes caused by advice from
other tasks. Before training, each run collects pilot rollouts on
training tasks with its initial advisor. At valid decisions where
advice was issued, we replace it in the with-advice scoring context
with advice from another task, which we call \emph{donor advice}.
Equation~\eqref{eq:score}, applied to the same recorded response and
no-advice baseline, then gives the donor contrast $d_j$. Donor advice
is scored but never sent to the executor. We set the threshold to
an empirical quantile of the $n_{\mathrm{cal}}$ donor contrast
magnitudes:

\begin{equation}
\epsilon_c=Q_{u_{\mathrm d}}\!\left(
\{|d_j|\}_{j=1}^{n_{\mathrm{cal}}}\right),
\qquad u_{\mathrm d}\in(0,1),
\label{eq:calibration}
\end{equation}
which stays fixed during training. Pilot scores for the issued advice
are used only for admission checks (Appendix~\ref{app:calibration}).
Of the flagged decisions $\mathcal J_i$, \method retains two kinds in
$\mathcal I_i$: original abstentions and decisions with
$|c_{i,k}|>\epsilon_c$. An original abstention has $C_k^+=C_k^-$ and
hence $c_k=0$, so the contrast cannot detect missed advice; flagged
abstentions therefore bypass scoring. A proposal to abstain after
issued advice must still pass the numeric gate.

\subsection{Self-distillation from targeted feedback}
\label{sec:objective}

At each retained decision, $\zeta_{i,k}$ combines local execution
evidence, relevant checks, episode score, and reflection feedback.
The teacher sees it prepended to the context $g_{i,k}$; the student
sees only $g_{i,k}$.
Both predict along the originally sampled advice $a_{i,k}$, including
\noadv{}:
\[
q_{i,k,t}=\sg\,\pi_{\bar\theta}
(\cdot\mid\zeta_{i,k}\oplus g_{i,k},a_{i,k,<t}),
\quad
p_{i,k,t}=\pi_\theta(\cdot\mid g_{i,k},a_{i,k,<t}).
\]
Here $\sg$ stops gradients, $\oplus$ prepends feedback, and both use
temperature $T_{\rm SD}$. At each prefix, both distributions are renormalized over a fixed
support $S$: the pre-update student's top-$K$ tokens.
\begin{equation}
\ell_{i,k}=\frac{1}{|a_{i,k}|}\sum_{t=1}^{|a_{i,k}|}
\KL(p^S_{i,k,t}\|q^S_{i,k,t}),\quad
\mathcal L=\mathcal L_{\rm base}+\frac{\lambda_s}{N_{\rm ep}}
\sum_{i=1}^{N_{\rm ep}}
\frac{\sum_{k\in\mathcal I_i^*}\ell_{i,k}}
{\max\{1,|\mathcal I_i^*|\}}.
\label{eq:loss}
\end{equation}
Here $\mathcal I_i^*\subseteq\mathcal I_i$ contains decisions with
feasible teacher contexts, $N_{\rm ep}$ counts episodes, and $\lambda_s$
weights distillation. We average token losses within each supervised decision, then average
these decision losses within each episode and the resulting episode
losses across the full batch. Episodes without supervision contribute
zero auxiliary loss. GRPO still uses every episode with its original
advantage. For each rollout batch, we compute teacher distributions
using the pre-update advisor and hold them fixed during optimization.
Self-distillation differentiates only through the student's
probabilities; token supports, recorded advice prefixes, and selection
weights also stay fixed (Appendix~\ref{app:teacher}).
\section{Targeted Supervision: Reward Learning and Calibration}
\label{sec:properties}\label{arc:sec:calibration}

Section~\ref{arc:sec:limit} studies distillation alone. We add reward
learning, where exposure (how often episodes receive supervision)
can also change the learning limit, and examine donor calibration.

\textbf{Selection alongside reward learning.}
We extend the two-teacher model of Section~\ref{arc:sec:limit} to
combine reward learning with distillation, and we keep the assumption
$\ell_I<\ell_S$. Reward learning encourages advice with higher
expected success, while distillation pulls the advisor toward the
retained teachers' mean target. We compare learning from all proposals
($j=0$) with selective retention ($j=G$). Both start from the same
initialization and use the same fixed weights:
\[
\dot\theta_j=F_j(\theta_j),\qquad
F_j(\theta)=c_0J'(\theta)-\lambda g_j(\theta),
\qquad c_0\ge0,\quad\lambda>0.
\]
Here $J'(\theta)$ is the gradient of expected success and $g_j(\theta)$
is the expected distillation gradient under retention rule $j$;
$c_0$ and $\lambda$ weight the two learning signals. The reward term
idealizes finite-step GRPO--AdamW training as exact gradient ascent.
Let $a_0$ denote the distillation-only equilibrium without gating
(Theorem~\ref{arc:thm:limit}).

\begin{arctheorem}[A higher learning limit under a common reward objective]
\label{arc:thm:combined}
Under the preceding two-teacher model, suppose $G$ independently
retains insensitive and sensitive proposals with fixed probabilities
$r_I^G,r_S^G\in(0,1]$, respectively, where $r_I^G<r_S^G$.
Without gating, every proposal is retained. From any common finite
initialization, both learning dynamics converge to finite equilibria
satisfying
\begin{equation}
\theta_G^\infty>\theta_0^\infty,\qquad
J(\theta_G^\infty)>J(\theta_0^\infty).
\label{arc:eq:limits}
\end{equation}
These inequalities hold even when the dynamics have multiple
equilibria. If the common initialization is at or above $a_0$, then
$\theta_G(t)>\theta_0(t)$ for every $t>0$.
\end{arctheorem}

Selection changes both \emph{which corrections are learned} and
\emph{how often learning from corrections occurs}. The mean retained
teacher target $\mu_j$ captures the first, and the probability $B_j$
that an episode receives supervision captures the second
(Section~\ref{arc:sec:limit}). At the same parameter value $\theta$,
the difference between the two updates is
\begin{equation}
\frac{F_G-F_0}{\lambda p(1-p)}
=\underbrace{B_G(\mu_G-\mu_0)}_{\text{composition}}
+\underbrace{(B_0-B_G)(\theta-\mu_0)}_{\text{exposure}}.
\label{arc:eq:joint-update}
\end{equation}
The composition term is positive because selection gives more weight
to the sensitive teacher, which assigns a higher probability to
advice 1 than the insensitive teacher. The exposure term comes from
less frequent supervision, and its sign depends on $\theta$.
Above $a_0$, ungated distillation pulls $\theta$ downward, so reducing
this pull helps. Below $a_0$, distillation pushes $\theta$ upward, so
reducing supervision can slow early progress. A higher learning limit
therefore need not mean faster learning from the start.

Both dynamics move upward below $a_0$, and $F_G>F_0$ at and above it.
Each trajectory remains bounded and cannot cross an equilibrium of its
own dynamics; together, these properties establish the ordered limits.
Appendices~\ref{arc:app:combined-two-teacher}
and~\ref{arc:app:combined} provide the proof, an example of slower
initial learning, and a stationary random-teacher extension.

For $c_0>0$, reward-only ascent approaches the model's best achievable
success. A fixed $\lambda>0$ instead produces a finite balance between
reward ascent and teacher fitting, and at that balance selection gives
higher success than no gating. The theorem compares the two combined
rules with each other; Section~\ref{sec:bfcl-results} tests gains over
outcome-only GRPO under finite training budgets.

Because selection also changes how often supervision occurs,
outperforming no gating does not by itself show that choosing
particular corrections helps. Retaining every proposal independently
with the same fixed probability preserves the teacher mixture and the
distillation-only limit, yet can improve eventual performance when
reward learning is present (Corollary~\ref{arc:cor:exposure}). This
motivates our matched-count random control
(Section~\ref{sec:ablations}). On each of the control's own episodes,
a shadow \method gate determines how many feasible, reflection-flagged
issued-advice decisions to retain. The control then randomly selects the same number of eligible
decisions, giving each decision an equal chance of being chosen. Unlike independent thinning, these
episode-dependent quotas need not preserve the ungated teacher
mixture (Appendix~\ref{app:selection-controls}).

\textbf{What donor calibration controls.}
Donor calibration sets a reference threshold using prediction changes
produced by advice from other tasks. The quantile in
Eq.~\eqref{eq:calibration} bounds how often donor advice passes the
gate on the pilot sample (Appendix~\ref{arc:app:pilot-count}).
We test whether the resulting selection rule improves learning,
alongside the contribution of the abstention bypass, through ablations
(Section~\ref{sec:ablations}). Appendix~\ref{app:supervision_dynamics}
reports which decisions receive supervision during training.

\setcounter{topnumber}{1}
\begin{table}[t]
\centering
\caption{\small \textbf{In-domain test performance.}
EnvScaler: native score $\times100$; BFCL-v3: official-checker accuracy
(\%) on 320 held-out tasks; Avg weights categories equally.
Mean $\pm$ sample SD follows Section~\ref{sec:setup-exp}.
Bold/underline: largest/second-largest distinct means, including ties,
per column/executor.}
\label{tab:bfcl}\label{tab:evaluation_results}
\setlength{\tabcolsep}{3pt}
\renewcommand{\arraystretch}{0.94}
\scriptsize
\resizebox{\linewidth}{!}{%
\begin{tabular}{@{}l|c|ccccc@{}}
\toprule
&\multicolumn{1}{c|}{EnvScaler}
&\multicolumn{5}{c}{BFCL-v3}\\
\cmidrule(lr){2-2}\cmidrule(lr){3-7}
Adaptation method&Score&Base&Miss Func&Miss Param&Long Ctx&Avg\\

\midrule
\rowcolor{GoogleBlue!10}
\multicolumn{7}{c}{\textbf{Frozen executor: Gemini 3.7 Flash}}\\
\multicolumn{7}{l}{\cellcolor{gray!12}\textit{Standalone executor: no advisor}}\\

No advisor
& $\resultcell{75.7}{2.0}$
& $\underline{\resultcell{70.8}{1.9}}$
& $\resultcell{60.4}{1.4}$
& $\resultcell{42.9}{1.9}$
& $\resultcell{57.1}{1.9}$
& $\resultcell{57.8}{1.4}$ \\

\midrule
\multicolumn{7}{l}{\cellcolor{gray!12}\textit{Executor prompt optimization: no advisor}}\\

GEPA \citep{agrawal2026gepa}
& $\resultcell{76.0}{1.8}$
& $\resultcell{69.6}{1.9}$
& $\resultcell{61.7}{1.4}$
& $\resultcell{45.4}{1.4}$
& $\resultcell{60.8}{1.4}$
& $\resultcell{59.4}{0.3}$ \\

\midrule
\multicolumn{7}{l}{\cellcolor{gray!12}\textit{Frozen advisors: no task-specific optimization}}\\

Gemini 3.7 Flash advisor
& $\resultcell{77.7}{0.7}$
& $\resultcell{70.4}{1.4}$
& $\resultcell{62.1}{1.9}$
& $\resultcell{44.6}{1.9}$
& $\resultcell{59.6}{1.4}$
& $\resultcell{59.2}{0.5}$ \\

Untrained Qwen3-8B
& $\resultcell{75.6}{2.2}$
& $\resultcell{67.1}{1.9}$
& $\resultcell{59.6}{1.9}$
& $\resultcell{39.6}{1.4}$
& $\resultcell{58.3}{1.9}$
& $\resultcell{56.1}{1.0}$ \\

\midrule
\multicolumn{7}{l}{\cellcolor{gray!12}\textit{Trained advisors: Qwen3-8B weight updates}}\\

GRPO \citep{shao2024deepseekmath}
& $\underline{\resultcell{79.7}{2.2}}$
& $\underline{\resultcell{70.8}{1.9}}$
& $\underline{\resultcell{64.2}{1.4}}$
& $\underline{\resultcell{47.1}{1.9}}$
& $\underline{\resultcell{63.3}{1.4}}$
& $\underline{\resultcell{61.4}{1.0}}$ \\

SDPO \citep{hubotter2026reinforcement}
& $\resultcell{76.0}{1.7}$
& $\resultcell{69.2}{1.9}$
& $\resultcell{60.8}{1.4}$
& $\resultcell{43.3}{1.4}$
& $\resultcell{57.9}{1.9}$
& $\resultcell{57.8}{0.8}$ \\

DistIL \citep{agrawal2026reinforcement}
& $\resultcell{76.1}{2.4}$
& $\resultcell{69.6}{1.4}$
& $\resultcell{61.7}{1.4}$
& $\resultcell{44.6}{1.4}$
& $\resultcell{58.3}{1.4}$
& $\resultcell{58.5}{0.4}$ \\

\rowcolor{GoogleGreen!9}
\method (ours)
& $\bestscore{\resultcell{84.8}{1.6}}$
& $\bestscore{\resultcell{73.3}{1.9}}$
& $\bestscore{\resultcell{72.1}{1.9}}$
& $\bestscore{\resultcell{54.2}{1.9}}$
& $\bestscore{\resultcell{71.7}{1.9}}$
& $\bestscore{\resultcell{67.8}{1.6}}$ \\

\midrule
\rowcolor{GoogleBlue!10}
\multicolumn{7}{c}{\textbf{Frozen executor: Claude Sonnet 4.6}}\\
\multicolumn{7}{l}{\cellcolor{gray!12}\textit{Standalone executor: no advisor}}\\

No advisor
& $\resultcell{77.4}{2.1}$
& $\resultcell{67.1}{1.4}$
& $\resultcell{53.8}{2.2}$
& $\resultcell{47.9}{1.4}$
& $\resultcell{50.8}{1.4}$
& $\resultcell{54.9}{0.9}$ \\

\midrule
\multicolumn{7}{l}{\cellcolor{gray!12}\textit{Executor prompt optimization: no advisor}}\\

GEPA \citep{agrawal2026gepa}
& $\resultcell{78.4}{1.3}$
& $\resultcell{67.1}{1.9}$
& $\resultcell{58.3}{1.9}$
& $\resultcell{56.3}{2.5}$
& $\resultcell{50.4}{1.4}$
& $\resultcell{58.0}{0.8}$ \\

\midrule
\multicolumn{7}{l}{\cellcolor{gray!12}\textit{Frozen advisors: no task-specific optimization}}\\

Claude Sonnet 4.6 advisor
& $\resultcell{78.7}{0.5}$
& $\underline{\resultcell{68.3}{1.9}}$
& $\resultcell{57.9}{1.9}$
& $\bestscore{\resultcell{62.1}{1.4}}$
& $\resultcell{53.3}{1.9}$
& $\underline{\resultcell{60.4}{0.7}}$ \\

Untrained Qwen3-8B
& $\resultcell{74.1}{1.5}$
& $\resultcell{63.3}{1.9}$
& $\resultcell{55.4}{1.9}$
& $\resultcell{47.1}{1.9}$
& $\resultcell{50.8}{1.9}$
& $\resultcell{54.2}{1.3}$ \\

\midrule
\multicolumn{7}{l}{\cellcolor{gray!12}\textit{Trained advisors: Qwen3-8B weight updates}}\\

GRPO \citep{shao2024deepseekmath}
& $\underline{\resultcell{80.7}{2.3}}$
& $\resultcell{66.7}{1.9}$
& $\underline{\resultcell{59.6}{1.4}}$
& $\resultcell{54.6}{1.4}$
& $\underline{\resultcell{54.6}{0.7}}$
& $\resultcell{58.9}{0.7}$ \\

SDPO \citep{hubotter2026reinforcement}
& $\resultcell{78.9}{1.2}$
& $\resultcell{64.2}{1.9}$
& $\resultcell{59.2}{1.4}$
& $\resultcell{49.2}{1.4}$
& $\resultcell{50.8}{1.4}$
& $\resultcell{55.8}{0.7}$ \\

DistIL \citep{agrawal2026reinforcement}
& $\resultcell{78.9}{1.5}$
& $\resultcell{65.4}{1.4}$
& $\underline{\resultcell{59.6}{1.9}}$
& $\resultcell{51.3}{2.2}$
& $\resultcell{52.9}{1.9}$
& $\resultcell{57.3}{1.3}$ \\

\rowcolor{GoogleGreen!9}
\method (ours)
& $\bestscore{\resultcell{84.6}{1.1}}$
& $\bestscore{\resultcell{69.6}{1.4}}$
& $\bestscore{\resultcell{65.4}{1.4}}$
& $\underline{\resultcell{59.2}{1.4}}$
& $\bestscore{\resultcell{58.3}{1.4}}$
& $\bestscore{\resultcell{63.1}{0.6}}$ \\

\bottomrule
\end{tabular}}
\end{table}
\section{Experiments}
\label{sec:experiments}

We ask four questions: (Q1) How does \method compare with baselines? (Q2) Which corrections should receive supervision? (Q3) Without retraining, does it transfer to out-of-domain tasks and across executor versions and families? (Q4) How does performance evolve during training?

\phantomsection\label{sec:setup-exp}
\textbf{Setup.}
We train Qwen3-8B advisors for two frozen executors, Gemini 3.7 Flash
and Claude Sonnet 4.6, separately on BFCL-v3 \citep{patil2025the}
and EnvScaler \citep{song2026envscaler}. The test sets are fixed
across runs: 320 BFCL tasks (80 in each of its four categories) and
200 EnvScaler tasks. For each seed, we re-split the remaining tasks
into training and validation sets (Appendix~\ref{app:baselines}).
Following LOPD \citep{zhang2026latent}, we report BFCL-v3 accuracy under the official multi-turn checker as an equal-weight average of
the four categories. 
EnvScaler uses native graded scores in $[0,1]$; we report their
mean multiplied by 100.
\method caps reflection at $b_{\rm refl}=5$ decisions per episode and calibrates each run's fixed threshold $\epsilon_c$ on a training-task
pilot using donor quantile $u_{\mathrm d}=0.95$.

\phantomsection\label{sec:reporting-exp}
\textbf{Evaluation.}
Every trained method has three independent training runs. For each
run, we take the checkpoint selected on validation, evaluate it four
times on the test set, and average the four scores. We report the
mean $\pm$ sample standard deviation (SD) of the three run averages.
GEPA is aggregated the same way, with three independent prompt
searches in place of training runs. The no-advisor and
frozen-advisor controls have no training runs, so their three
averages come from three groups of four evaluations of the same
system. Their SD reflects only evaluation randomness
(Appendix~\ref{app:statistics}).

\textbf{Baselines.}
The \emph{standalone executor} receives no advice.
For \emph{executor prompt optimization}, GEPA
\citep{agrawal2026gepa} optimizes reusable executor instructions
without an advisor.
The \emph{frozen advisors} are untrained Qwen3-8B and the
executor's own API model; neither receives task-specific optimization.
Among \emph{trained advisors}, outcome-only advisor-GRPO
\citep{shao2024deepseekmath}, labeled GRPO in the tables, uses the
same reward-learning objective as \method without self-distillation.
Adapted SDPO \citep{hubotter2026reinforcement} and
DistIL \citep{agrawal2026reinforcement} use successful sibling
rollouts of the same task as feedback, using their
distillation objectives without an added GRPO loss.
All advised systems share the response-level interface and can abstain
(Appendices~\ref{app:configuration} and~\ref{app:baselines}).

\phantomsection\label{sec:bfcl-results}
\textbf{(Q1) In-domain performance.}
\method has the highest BFCL-v3 average and EnvScaler score for
both executors (Table~\ref{tab:bfcl}). It exceeds GRPO by
4.2--6.4 percentage points on BFCL-v3 and 3.9--5.1 score points on
EnvScaler. For both executors, gains over GRPO are larger in BFCL's
augmented categories than in Base, suggesting greater benefits when
tools (Missing Function) or required information (Missing Parameter)
are missing, or contexts are longer (Long Context). Untrained Qwen3-8B advice lowers all four aggregate means relative
to standalone execution, while frozen frontier advisors raise them.
After \method training, the smaller Qwen3-8B advisor outperforms
these frontier advisors in every comparison in Table~\ref{tab:bfcl}
except Missing Parameter with the Claude executor. Adapted SDPO and
DistIL reuse the same successful-rollout feedback to supervise the
advisor's generated tokens across turns, yet both trail GRPO in
aggregate. This gap motivates decision-level targeting, since shared
trajectory-level feedback need not be equally useful throughout a
multi-turn interaction. Appendix~\ref{app:cases} shows a case study comparing
execution with and without a trained advisor.

\begin{table}[t]
\centering
\caption{\small
\textbf{Correction selection and abstention supervision (Q2).}
All \method variants share GRPO, reflection, teacher construction,
and the auxiliary loss and weight schedule. Self-distillation
uses feasible reflection-flagged decisions.
\emph{\method} retains issued-advice decisions whose with- and
without-advice scores differ by more than the threshold in either
direction; \emph{inverted gate} retains those with absolute score
differences at or below the threshold, without count matching.
\emph{No gate} retains every proposal.
\emph{Matched-count random} samples issued-advice decisions
uniformly, matching the gate's per-episode count on the random
control's own rollouts. All four retain flagged original abstentions.
\emph{No abstention bypass} removes self-distillation at original
abstentions while keeping the ordinary gate, GRPO, and the advisor's
ability to abstain. GRPO alone uses no self-distillation.
Metrics and averaging follow Table~\ref{tab:bfcl};
bold marks column maxima.}
\label{tab:ablation}
\small
\setlength{\tabcolsep}{4pt}
\renewcommand{\arraystretch}{0.95}
\begin{tabularx}{\linewidth}{@{}l*{4}{>{\centering\arraybackslash}X}@{}}
\toprule
&\multicolumn{2}{c}{BFCL-v3}
&\multicolumn{2}{c}{EnvScaler}\\
\cmidrule(lr){2-3}\cmidrule(lr){4-5}
Advisor training
&{\fontsize{8}{9}\selectfont\mbox{Gemini 3.7 Flash}}
&{\fontsize{8}{9}\selectfont\mbox{Claude Sonnet 4.6}}
&{\fontsize{8}{9}\selectfont\mbox{Gemini 3.7 Flash}}
&{\fontsize{8}{9}\selectfont\mbox{Claude Sonnet 4.6}}\\
\midrule
GRPO
&$61.4\pm1.0$
&$58.9\pm0.7$
&$79.7\pm2.2$
&$80.7\pm2.3$\\
\addlinespace[2pt]
\method, no gate
&$62.7\pm0.8$
&$59.7\pm0.3$
&$81.1\pm1.5$
&$81.4\pm0.9$\\
\method, inverted gate
&$61.1\pm1.4$
&$58.3\pm0.6$
&$78.7\pm1.9$
&$79.4\pm1.6$\\
\method, matched-count random
&$62.9\pm0.9$
&$60.6\pm0.4$
&$80.9\pm1.3$
&$81.1\pm1.2$\\
\rowcolor{GoogleGreen!9}
\method
&$\mathbf{67.8\pm1.6}$
&$\mathbf{63.1\pm0.6}$
&$\mathbf{84.8\pm1.6}$
&$\mathbf{84.6\pm1.1}$\\
\midrule
\method, no abstention bypass
&$66.1\pm1.3$
&$62.5\pm0.8$
&$83.3\pm2.1$
&$83.5\pm1.5$\\
\bottomrule
\end{tabularx}
\end{table}

\begin{table}[t]
\centering
\caption{\small \textbf{Out-of-domain (OOD) performance without retraining (Q3).}
BFCL-selected systems; mean $\pm$ sample SD follows the
\hyperref[sec:reporting-exp]{evaluation protocol}. Metrics (\%): ACEBench success, ToolHop
answer correctness (AC), $\tau^2$-bench pass$^1$, and RoTBench tool
selection (TS), parameter identification (PI), and content filling
(CF). Macro Avg averages components within each benchmark, then the
four benchmarks equally (Appendix~\ref{app:statistics}). Bold/underline mark the
largest/second-largest distinct means per column and executor,
including ties.}
\label{tab:ood}
\setlength{\tabcolsep}{2.1pt}
\scriptsize
\resizebox{\linewidth}{!}{%
\begin{tabular}{@{}l|cc|c|ccc|ccc|c@{}}
\toprule
&\multicolumn{2}{c|}{ACEBench}
&\multicolumn{1}{c|}{ToolHop}
&\multicolumn{3}{c|}{$\tau^2$-bench}
&\multicolumn{3}{c|}{RoTBench}&\\
\cmidrule(lr){2-3}
\cmidrule(lr){4-4}
\cmidrule(lr){5-7}
\cmidrule(lr){8-10}
\multirow{-2}{*}{Adaptation method}
&M-Step&M-Turn&AC&Airline&Retail&Telecom&TS&PI&CF
&\multirow{-2}{*}{\shortstack{Macro\\Avg}}\\

\midrule
\rowcolor{GoogleBlue!10}
\multicolumn{11}{c}{\textbf{Frozen executor: Gemini 3.7 Flash}}\\
\multicolumn{11}{l}{\cellcolor{gray!12}\textit{Standalone executor: no advisor}}\\

No advisor
& $\resultcell{86.7}{2.9}$
& $\resultcell{66.7}{5.8}$
& $\underline{\resultcell{81.3}{0.4}}$
& $\resultcell{84.0}{2.0}$
& $\resultcell{58.2}{1.3}$
& $\underline{\resultcell{91.5}{1.3}}$
& $\underline{\resultcell{71.9}{0.8}}$
& $\underline{\resultcell{66.1}{0.4}}$
& $\underline{\resultcell{47.8}{0.2}}$
& $74.5$ \\

\multicolumn{11}{l}{\cellcolor{gray!12}\textit{Executor prompt optimization: no advisor}}\\

GEPA
& $\resultcell{85.0}{2.2}$
& $\resultcell{63.9}{2.1}$
& $\resultcell{80.6}{0.1}$
& $\resultcell{80.0}{3.5}$
& $\resultcell{55.8}{1.3}$
& $\resultcell{88.6}{0.9}$
& $\resultcell{68.8}{0.4}$
& $\resultcell{63.4}{0.3}$
& $\resultcell{45.9}{0.2}$
& $72.3$ \\

\multicolumn{11}{l}{\cellcolor{gray!12}\textit{Frozen advisors: no task-specific optimization}}\\

Gemini 3.7 Flash advisor
& $\underline{\resultcell{91.7}{2.9}}$
& $\bestscore{\resultcell{70.0}{5.8}}$
& $\resultcell{80.4}{0.2}$
& $\underline{\resultcell{89.3}{2.3}}$
& $\resultcell{58.8}{2.3}$
& $\resultcell{83.6}{1.0}$
& $\resultcell{67.3}{0.2}$
& $\resultcell{61.1}{0.4}$
& $\resultcell{45.7}{0.4}$
& $74.1$ \\

Untrained Qwen3-8B
& $\resultcell{85.0}{5.0}$
& $\resultcell{61.1}{7.7}$
& $\resultcell{81.1}{0.3}$
& $\resultcell{82.0}{2.0}$
& $\resultcell{55.8}{2.0}$
& $\resultcell{83.3}{1.8}$
& $\resultcell{70.2}{0.4}$
& $\resultcell{65.1}{0.3}$
& $\resultcell{46.9}{0.3}$
& $72.1$ \\

\multicolumn{11}{l}{\cellcolor{gray!12}\textit{Trained advisors: Qwen3-8B weight updates}}\\

GRPO
& $\bestscore{\resultcell{93.3}{2.9}}$
& $\underline{\resultcell{68.9}{1.9}}$
& $\resultcell{80.5}{0.3}$
& $\resultcell{86.7}{4.2}$
& $\underline{\resultcell{59.1}{2.7}}$
& $\resultcell{90.4}{0.9}$
& $\resultcell{70.0}{0.2}$
& $\resultcell{63.8}{0.2}$
& $\resultcell{47.2}{0.4}$
& $\underline{75.2}$ \\

SDPO
& $\resultcell{88.3}{2.9}$
& $\resultcell{65.6}{8.4}$
& $\resultcell{79.5}{0.2}$
& $\resultcell{82.7}{4.2}$
& $\resultcell{54.1}{4.1}$
& $\resultcell{86.5}{2.0}$
& $\resultcell{66.8}{0.4}$
& $\resultcell{62.9}{0.4}$
& $\resultcell{45.3}{0.5}$
& $72.3$ \\

DistIL
& $\underline{\resultcell{91.7}{2.9}}$
& $\resultcell{64.4}{5.1}$
& $\resultcell{80.6}{0.1}$
& $\resultcell{80.7}{2.3}$
& $\resultcell{58.5}{2.2}$
& $\resultcell{86.3}{1.8}$
& $\resultcell{68.9}{0.9}$
& $\resultcell{63.1}{0.6}$
& $\resultcell{46.3}{0.2}$
& $73.3$ \\

\rowcolor{GoogleGreen!9}
\method (ours)
& $\bestscore{\resultcell{93.3}{2.9}}$
& $\bestscore{\resultcell{70.0}{5.8}}$
& $\bestscore{\resultcell{82.3}{0.3}}$
& $\bestscore{\resultcell{91.3}{1.2}}$
& $\bestscore{\resultcell{62.0}{2.0}}$
& $\bestscore{\resultcell{92.4}{1.8}}$
& $\bestscore{\resultcell{73.0}{0.7}}$
& $\bestscore{\resultcell{67.3}{0.4}}$
& $\bestscore{\resultcell{49.1}{0.3}}$
& $\bestscore{77.2}$ \\

\midrule
\rowcolor{GoogleBlue!10}
\multicolumn{11}{c}{\textbf{Frozen executor: Claude Sonnet 4.6}}\\
\multicolumn{11}{l}{\cellcolor{gray!12}\textit{Standalone executor: no advisor}}\\

No advisor
& $\bestscore{\resultcell{86.7}{5.8}}$
& $\resultcell{66.7}{3.3}$
& $\underline{\resultcell{67.0}{0.1}}$
& $\underline{\resultcell{87.3}{1.2}}$
& $\resultcell{57.9}{2.3}$
& $\underline{\resultcell{92.7}{0.5}}$
& $\resultcell{74.2}{0.7}$
& $\resultcell{58.9}{0.4}$
& $\resultcell{42.9}{0.5}$
& $70.4$ \\

\multicolumn{11}{l}{\cellcolor{gray!12}\textit{Executor prompt optimization: no advisor}}\\

GEPA
& $\underline{\resultcell{85.0}{8.7}}$
& $\resultcell{64.4}{5.1}$
& $\resultcell{65.8}{0.4}$
& $\resultcell{79.3}{3.1}$
& $\resultcell{55.8}{2.0}$
& $\resultcell{87.1}{2.5}$
& $\resultcell{73.8}{0.4}$
& $\resultcell{58.5}{0.7}$
& $\resultcell{42.5}{0.7}$
& $68.2$ \\

\multicolumn{11}{l}{\cellcolor{gray!12}\textit{Frozen advisors: no task-specific optimization}}\\

Claude Sonnet 4.6 advisor
& $\resultcell{81.7}{2.9}$
& $\underline{\resultcell{68.9}{7.7}}$
& $\resultcell{66.1}{0.4}$
& $\resultcell{84.0}{2.0}$
& $\resultcell{58.5}{1.3}$
& $\resultcell{87.7}{2.3}$
& $\resultcell{77.7}{0.7}$
& $\underline{\resultcell{66.9}{0.5}}$
& $\underline{\resultcell{50.1}{0.8}}$
& $70.8$ \\

Untrained Qwen3-8B
& $\resultcell{81.7}{7.6}$
& $\resultcell{64.4}{6.9}$
& $\resultcell{66.7}{0.5}$
& $\resultcell{77.3}{1.2}$
& $\resultcell{51.8}{1.5}$
& $\resultcell{86.0}{0.9}$
& $\resultcell{73.0}{0.5}$
& $\resultcell{57.3}{0.7}$
& $\resultcell{41.2}{0.4}$
& $67.2$ \\

\multicolumn{11}{l}{\cellcolor{gray!12}\textit{Trained advisors: Qwen3-8B weight updates}}\\

GRPO
& $\bestscore{\resultcell{86.7}{5.8}}$
& $\resultcell{67.8}{1.9}$
& $\resultcell{66.5}{0.3}$
& $\resultcell{82.7}{4.6}$
& $\underline{\resultcell{59.6}{2.3}}$
& $\resultcell{90.6}{2.5}$
& $\underline{\resultcell{77.9}{1.0}}$
& $\resultcell{65.4}{0.4}$
& $\resultcell{48.3}{0.4}$
& $\underline{71.3}$ \\

SDPO
& $\underline{\resultcell{85.0}{8.7}}$
& $\resultcell{66.7}{3.3}$
& $\resultcell{65.4}{0.2}$
& $\resultcell{76.7}{3.1}$
& $\resultcell{57.0}{1.5}$
& $\resultcell{86.8}{2.6}$
& $\resultcell{74.3}{0.8}$
& $\resultcell{62.4}{0.1}$
& $\resultcell{45.9}{0.3}$
& $68.9$ \\

DistIL
& $\resultcell{83.3}{10.4}$
& $\resultcell{61.1}{5.1}$
& $\resultcell{66.1}{0.3}$
& $\resultcell{79.3}{4.6}$
& $\resultcell{56.1}{3.8}$
& $\resultcell{88.6}{2.3}$
& $\resultcell{75.1}{0.2}$
& $\resultcell{62.1}{0.2}$
& $\resultcell{45.0}{0.9}$
& $68.4$ \\

\rowcolor{GoogleGreen!9}
\method (ours)
& $\underline{\resultcell{85.0}{5.0}}$
& $\bestscore{\resultcell{70.0}{5.8}}$
& $\bestscore{\resultcell{68.8}{0.5}}$
& $\bestscore{\resultcell{88.0}{3.5}}$
& $\bestscore{\resultcell{64.9}{2.3}}$
& $\bestscore{\resultcell{93.3}{1.3}}$
& $\bestscore{\resultcell{79.8}{0.4}}$
& $\bestscore{\resultcell{70.0}{0.3}}$
& $\bestscore{\resultcell{52.6}{0.1}}$
& $\bestscore{74.0}$ \\

\bottomrule
\end{tabular}}
\end{table}

\begin{figure}[!b]
\centering
\includegraphics[width=\linewidth]{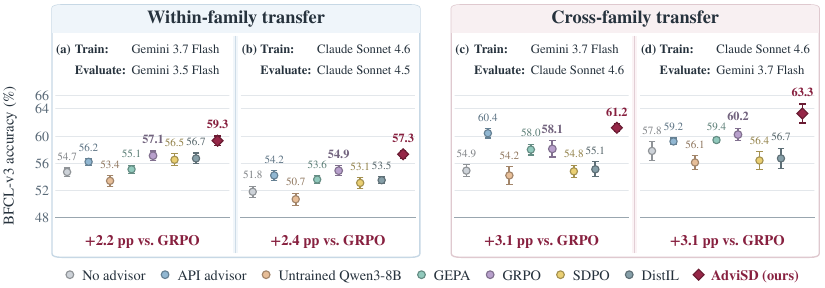}
\caption{\small
\textbf{Cross-version and cross-family BFCL-v3 transfer without retraining (Q3).}
Train/Evaluate labels identify executors; API advisors use the
evaluation executor's model.
Dots show means; intervals indicate $\pm$ sample SD.
Gains show \method minus GRPO in \% points.
All panels share a 48--66\% accuracy scale.}
\label{fig:transfer}
\end{figure}

\phantomsection\label{sec:ablations}
\textbf{(Q2) Which corrections should the advisor learn?}
\method exceeds matched-count random selection by 2.5--4.9 points
and no gating by 3.2--5.1 (Table~\ref{tab:ablation}), whereas
the two controls differ by at most 0.9 points.
Theorem~\ref{arc:thm:combined} and Corollary~\ref{arc:cor:exposure}
motivate the random control, which matches the gate's per-episode
counts on its own rollouts. \method is trained separately, so its
episodes and supervision counts can differ from the control's.
The inverted gate also trails GRPO on both benchmarks with both
executors. These results support the predictive selection rule:
randomly reducing supervision does not recover \method's gains, and
prioritizing low-contrast decisions does worse than reward learning
alone. Next, at an original abstention, no advice is issued, so the with- and
without-advice contexts are identical and $c_k=0$. Reflection may
still flag missed advice, in which case the bypass lets the
abstention receive self-distillation. Removing the bypass lowers
performance by 0.6--1.7 points across both benchmarks and executors.
The no-bypass variant, which selects only among issued-advice
decisions, still exceeds GRPO in every aggregate.
Appendix~\ref{app:supervision_dynamics} reports supervision
allocation in one Claude BFCL run.

\phantomsection\label{sec:ood-results}
\textbf{(Q3) Transfer without retraining.}
We evaluate the BFCL-selected advisors and prompts on ACEBench
\citep{chen2025acebench}, ToolHop \citep{ye2025toolhop},
$\tau^2$-bench \citep{barres2026taubench}, and RoTBench
\citep{ye2024rotbench} without further adaptation, using each
benchmark's native tools, policies, and metrics. \method's gains
carry over to these benchmarks: it leads the four-benchmark
macro-average for both executors (Table~\ref{tab:ood}). It exceeds
standalone execution by 2.7 points with Gemini and 3.6 with Claude,
while GRPO's margin is less than one point. \method has the highest
mean, outright or tied, in 17 of 18 comparisons. By contrast, BFCL-optimized GEPA prompts fall below standalone
execution on every evaluated component, indicating its poor cross-benchmark transfer. The effect of advice still varies by task. On $\tau^2$-bench
telecom, the frozen Gemini advisor scores 7.9\% points
below standalone execution, whereas \method scores slightly above
it. \method itself trails standalone execution and GRPO by
1.7 \% points on Claude's ACEBench multi-step tasks. Useful guidance also transfers across executor versions and
families. \method leads all four transfer comparisons
(Figure~\ref{fig:transfer}) and exceeds transferred GRPO by
3.1 percentage points in both cross-family directions. Its
cross-family scores still fall below those of \method advisors
trained for the evaluation executor (Table~\ref{tab:bfcl}), so
transferred model does not replace executor-specific
training.

\begin{figure}[t]
\centering
\includegraphics[width=0.5\linewidth]{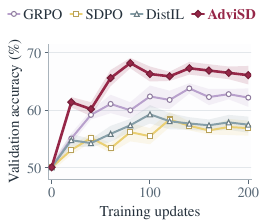}
\caption{\small
\textbf{Training progress on BFCL-v3.}
Validation accuracy across training updates with Claude Sonnet 4.6
as the frozen executor. Mean $\pm$ sample standard deviation over
three independent training runs. \method stays ahead of the baselines
from update 20 onward.}
\label{fig:validation}
\end{figure}

\phantomsection\label{sec:training-results}
\textbf{(Q4) Training progress.}
On BFCL-v3 validation with Claude Sonnet 4.6, \method leads the
compared methods by update 20 and remains ahead at every later
evaluated checkpoint. Late in training, its margin over GRPO is
roughly four percentage points (Figure~\ref{fig:validation}).

\section{Conclusion}
\label{sec:conclusion}

\method combines task rewards with targeted self-distillation to train
small advisors for frozen frontier executors. It separates proposing
corrections from selecting which advice decisions to supervise.
Selection compares how the advisor scores the same recorded executor
response with and without its issued advice, and a
feedback-conditioned teacher supervises the retained decisions.
Our shared-parameter analysis explains how eventual performance
depends on which corrections are retained and how often supervision
occurs, and it motivates the matched-count random control.
\method has the highest in-domain aggregates on BFCL-v3 and EnvScaler
among the compared methods. Without retraining, it also leads the
out-of-domain macro-averages and transfer comparisons across executor
versions and model families. Its gains over ungated and matched-count
random supervision indicate that choosing where to learn matters
beyond reducing supervision.
Appendix~\ref{app:limitations} discusses theoretical scope, the
predictive score, and evaluation limitations.

\bibliography{main}

\appendix
\raggedbottom
\makeatletter
\let\ARCappendixaddcontentsline\addcontentsline
\renewcommand{\addcontentsline}[3]{%
  \def\ARCentryfile{#1}\def\ARCtocfile{toc}%
  \ifx\ARCentryfile\ARCtocfile
    \ARCappendixaddcontentsline{toc}{#2}{#3}
    \ARCappendixaddcontentsline{atoc}{#2}{#3}%
  \else
    \ARCappendixaddcontentsline{#1}{#2}{#3}%
  \fi}
\pdfbookmark[0]{Appendix contents}{arc-appendix-contents}

\newpage
\section*{Appendices}
\begingroup
  \setcounter{tocdepth}{2}
  \@starttoc{atoc}
\endgroup
\hypersetup{linktoc=section}
\makeatother
\clearpage
\section{Notation and scope}
\label{arc:app:guide}\label{app:notation}\label{app:runtime-notation}

Logarithms are natural, and vector norms are Euclidean. We compute an
expected training gradient by differentiating the student loss with the
sampled data, teacher, supports, and selection weights held fixed, and
then averaging over the sampling law; the rollout distribution is not
differentiated. In the local analysis, $\theta$ is the advisor's
parameter vector. In the learning model, it is a single scalar logit
shared across situations. That model assumes independent opportunities,
uniform proposals, and stationary conditional statistics. It explains a
selection mechanism and does not describe the full dynamics of the
reflector, the evolving language model, or the optimizer.

\begin{center}\small
\renewcommand{\arraystretch}{1.08}
\begin{tabularx}{\linewidth}{@{}p{.14\linewidth}Y@{}}
\toprule
Symbol & Meaning\\
\midrule
\multicolumn{2}{@{}l}{\textit{Local update analysis}}\\
\addlinespace[2pt]
$h$ & Advice prefix at a fixed interaction state.\\
$S_h$ & Fixed finite set of candidate tokens at prefix $h$.\\
$p_\theta$ & Student next-token distribution, normalized on $S_h$.\\
$q_h$ & Teacher next-token distribution on $S_h$, held fixed during differentiation.\\
$K_{h,v}$ & Law of executor responses and outcomes after token $v$ and completion by the pre-update advisor; excludes the advice text.\\
$V_h(v)$ & Expected task score under the execution law $K_{h,v}$.\\
$J_h(\theta)$ & Local execution value: $V_h(v)$ averaged under the student's next-token distribution.\\
$g_h$ & Student-to-teacher reverse-KL gradient evaluated at the pre-update parameters.\\
$e_h$ & Teacher-dependent residual relative to the value-tilted reference, in $g_h=-\alpha\nabla J_h+e_h$.\\
\midrule
\multicolumn{2}{@{}l}{\textit{Repeated-learning model}}\\
\addlinespace[2pt]
$M$ & Number of independent opportunities per modeled episode.\\
$b$ & Maximum number of failure proposals per modeled episode.\\
$\beta$ & Probability that an opportunity is execution-insensitive.\\
$r_{ja}$ & Retention probability for a proposed failure of type $j$ with issued action $a$; excludes proposal sampling.\\
$\ell_{ja}$ & Expected teacher log-odds for type $j$ and issued action $a$, conditional on failure and retention.\\
$B(\theta)$ & Exposure: probability that an episode retains at least one correction.\\
$\mu(\theta)$ & Composition: expected within-episode mean of retained teacher log-odds, conditional on retaining at least one correction.\\
\midrule
\multicolumn{2}{@{}l}{\textit{\method contexts and selection}}\\
\addlinespace[2pt]
$b_{\rm refl}$ & Maximum number of advice decisions flagged by reflection per eligible episode; distinct from the modeled proposal cap $b$.\\
$g_k$ & Advisor context before decision $k$: visible interaction, tool schemas, and earlier advice.\\
$h_k$ & Native executor history before response $k$; distinct from the theoretical advice prefix $h$.\\
$C_k^+$ & Paired scoring context containing pre-response history, tool schemas, and issued advice.\\
$C_k^-$ & The same scoring context with the issued advice omitted.\\
$\zeta_{i,k}$ & Feedback block with local execution evidence, relevant checks, episode score, and reflection feedback, prepended to the teacher's context at decision $k$ in episode $i$.\\
$c_k$ & Signed mean token log-score contrast for the same recorded executor response, with issued advice versus no advice.\\
$d_j$ & Corresponding contrast for pilot pair $j$, replacing issued advice with advice borrowed from another task.\\
$\epsilon_c$ & Per-run fixed selection threshold, calibrated from an empirical quantile of absolute pilot-donor contrasts $|d_j|$.\\
\bottomrule
\end{tabularx}
\end{center}

\section{A single update through the executor}
\label{arc:app:channel}
\label{arc:app:graded}

Lemma~\ref{arc:lem:channel} separates fitting a feedback-conditioned
teacher from improving execution. The two differ because the
distillation loss compares advice-token distributions, whereas task
performance depends on how the executor responds to the completed
advice. We first express both at one recorded prefix so that they can
be compared directly. This gives the gradient identity in
Section~\ref{arc:sec:channel} and an extension to situations where
execution depends only weakly on the next token. The extension matters because exact insensitivity is a boundary case:
it quantifies how weak token--execution dependence limits the value
component when the reward range and gradient scale are controlled.

Fix an interaction situation $\sigma$, including the environment state
and the advisor's and executor's histories before advice. For complete
advice $a$, let $K_\sigma(a)$ denote the law of the subsequent executor
response, tool outcomes, transitions, and termination. The advice text
itself is excluded from this law; otherwise changing the advice would
change the recorded outcome even if the executor behaved identically.
At a recorded advice prefix $h$, choose a next token $v$ and complete
the advice using the pre-update advisor. If $C_{h,v}$ is the resulting
distribution over complete advice strings, the induced execution law is
\[
K_{h,v}=\mathbb E_{a\sim C_{h,v}}K_\sigma(a).
\]
Thus $K_{h,v}$ already averages over stochastic advice completion and
execution. For a bounded task score $W\in[w_-,w_+]$, define
\begin{equation}
V_h(v)=\mathbb E_{K_{h,v}}W,\qquad
J_h(\theta)=\sum_{v\in S_h}p_\theta(v)V_h(v).
\label{arc:eq:channel}
\end{equation}
The first quantity is the expected score after choosing $v$; the second
averages these values using the student's next-token probabilities.
Only this next-token distribution varies in $J_h$. Holding the
completion and execution laws fixed allows us to ask what changing the
token preference alone would accomplish, without also changing the
policies used for later decisions.

\begin{arcassumption}[Fixed-prefix conditional update]
\label{arc:ass:update}
The finite support $S_h$, context, prefix, completion and execution
laws, objective, and loss weights are fixed during differentiation.
Student logits are differentiable near the pre-update parameters
$\bar\theta$; the student $p_\theta$ and detached teacher $q_h$ are
positive and normalized on $S_h$. If the teacher is random,
$\mathbb E[-\log q_h(v)]<\infty$ for each supported token.
\end{arcassumption}

The support can be the full vocabulary or a restricted set fixed for
the update, as in Section~\ref{sec:objective}. Positivity ensures that
the log ratios and reverse KL are finite, and finite softmax logits
satisfy it. A token whose student probability is identically zero
in a neighborhood of $\bar\theta$ can be omitted from the support;
in contrast, a zero teacher probability at positive student mass makes
the reverse KL infinite and is not covered here.

Write $\bar p=p_{\bar\theta}$,
$\phi_h(v)=\nabla_\theta\log p_\theta(v)|_{\bar\theta}$, and
$R_W=w_+-w_-$. The vector $\phi_h(v)$ describes how the student's
log-probability of $v$ changes with its parameters. To compare a
feedback-conditioned teacher with execution value, introduce the reference
distribution
\[
q_h^\alpha(v)=\frac{\bar p(v)e^{\alpha V_h(v)}}{Z_\alpha},\qquad
Z_\alpha=\sum_{u\in S_h}\bar p(u)e^{\alpha V_h(u)},\quad \alpha\ge0.
\]
For positive $\alpha$, this distribution shifts probability toward
higher-value tokens, but only as far as a penalty for departing from
the current student allows. Indeed, for any distribution $q$
on $S_h$,
\[
\mathbb E_qV_h-\alpha^{-1}\KL(q\|\bar p)
=\alpha^{-1}\log Z_\alpha-\alpha^{-1}\KL(q\|q_h^\alpha).
\]
The right-hand side is uniquely maximized at $q=q_h^\alpha$.
At $\alpha=0$, the definition reduces to $q_h^0=\bar p$.
\method never constructs this reference teacher. We use it only as a
comparison with a known relation to execution value, which shows what
changes when the actual feedback-conditioned teacher is used instead.

For the weak-dependence statement, define
\[
F_h=\mathbb E_{\bar p}\|\phi_h\|_2^2,\qquad
I_h=\operatorname{I}(X;Z),\quad X\sim\bar p,\quad Z\sim K_{h,X}.
\]
Here $F_h$ controls the scale of the log-probability gradients, and
$I_h$ measures how informative the sampled token is about the ensuing
execution. Both belong to the conditional model just defined; neither
is the predictive contrast used by \method.

\begin{arcrestatedlemma}[Execution-value component and teacher residual]
\label{arc:lem:channel-formal}
Under Assumption~\ref{arc:ass:update}, let
$g_h=\nabla_\theta\KL(p_\theta\|q_h)|_{\bar\theta}$.
For every fixed $\alpha\ge0$,
\begin{align}
g_h&=-\alpha\nabla J_h(\bar\theta)+e_h,\qquad
e_h=\mathbb E_{\bar p}\left[\phi_h\log\frac{q_h^\alpha}{q_h}\right],
\label{arc:eq:decomposition-formal}\\
\|g_h-e_h\|_2&=\alpha\|\nabla J_h(\bar\theta)\|_2
\le\alpha R_W\sqrt{F_hI_h/2}.
\label{arc:eq:graded-formal}
\end{align}
If all supported execution laws $K_{h,v}$ coincide, then
\[
\nabla J_h(\bar\theta)=0,\qquad q_h^\alpha=\bar p,\qquad g_h=e_h.
\]
The identities hold for each realized teacher and may be averaged
over its law.
\end{arcrestatedlemma}

\begin{proof}
We begin with the gradient of the conditional distillation loss.
Differentiating $\sum_v p_\theta(v)=1$ gives
\[
\mathbb E_{\bar p}\phi_h
=\sum_{v\in S_h}\nabla_\theta p_\theta(v)|_{\bar\theta}=0.
\]
Because the teacher is held fixed, differentiating the reverse KL
gives
\[
\begin{aligned}
g_h
&=\sum_{v\in S_h}\bar p(v)\phi_h(v)
  \left[\log\frac{\bar p(v)}{q_h(v)}+1\right]\\
&=\mathbb E_{\bar p}\left[\phi_h\log\frac{\bar p}{q_h}\right].
\end{aligned}
\]
The reference teacher lets us rewrite the log ratio as
\[
\log\frac{\bar p(v)}{q_h(v)}
=\log\frac{q_h^\alpha(v)}{q_h(v)}-\alpha V_h(v)+\log Z_\alpha.
\]
The last term is constant across tokens, so its contribution vanishes
by the zero-mean score identity. Moreover, the fixed execution values
in Eq.~\eqref{arc:eq:channel} satisfy
\[
\nabla J_h(\bar\theta)
=\sum_v\bar p(v)\phi_h(v)V_h(v)
=\mathbb E_{\bar p}[\phi_hV_h].
\]
Substituting these two facts gives
$g_h=e_h-\alpha\nabla J_h(\bar\theta)$, which is the identity in
Eq.~\eqref{arc:eq:decomposition-formal}. If all $K_{h,v}$ coincide,
all tokens have the same value. The value gradient is then zero,
and the common exponential factor cancels in $q_h^\alpha$.
Consequently $q_h^\alpha=\bar p$ and $g_h=e_h$.

To obtain the quantitative bound, consider the execution mixture
$\nu_h=\sum_v\bar p(v)K_{h,v}$. Its expected score is
$\mathbb E_{\nu_h}W=J_h(\bar\theta)$. Centering the value in the
gradient is permissible because $\mathbb E_{\bar p}\phi_h=0$, giving
\[
\nabla J_h(\bar\theta)
=\mathbb E_{\bar p}\!\left[
\phi_h(v)\bigl(V_h(v)-J_h(\bar\theta)\bigr)\right].
\]
The vector Cauchy--Schwarz inequality now gives
\begin{equation}
\|\nabla J_h(\bar\theta)\|_2^2
\le F_h\operatorname{Var}_{\bar p}(V_h).
\label{arc:eq:graded}
\end{equation}
It remains to relate the variation in token values to the variation in
their execution laws. Because $W$ has range $R_W$, the
bounded-function characterization of total variation gives
\[
|V_h(v)-J_h(\bar\theta)|
\le R_W\operatorname{TV}(K_{h,v},\nu_h).
\]
Pinsker's inequality then implies
\[
\bigl(V_h(v)-J_h(\bar\theta)\bigr)^2
\le \frac{R_W^2}{2}\KL(K_{h,v}\|\nu_h).
\]
Averaging over the token choice and using the conditional-law
expression for mutual information yields
\[
\operatorname{Var}_{\bar p}(V_h)
\le\frac{R_W^2}{2}
\sum_v\bar p(v)\KL(K_{h,v}\|\nu_h)
=\frac{R_W^2 I_h}{2}.
\]
Combining this with Eq.~\eqref{arc:eq:graded} proves
Eq.~\eqref{arc:eq:graded-formal}. All terms are finite:
$\nu_h\ge\bar p(v)K_{h,v}$ implies
$\KL(K_{h,v}\|\nu_h)\le-\log\bar p(v)$, and the support is finite
and positive. Finally, the teacher log-moment assumption makes each
coordinate of the teacher-dependent gradient integrable, which permits
the stated averaging over a random teacher.
\end{proof}

The identity explains what fitting the reference teacher would do:
its gradient-descent direction is $\alpha\nabla J_h$. Fitting another
teacher adds the residual direction $-e_h$, which can reinforce or
oppose value ascent. When execution is insensitive to the next token,
there is no local value component at all, but teacher fitting can still
move the advisor. The bound extends this observation to execution that
is only nearly insensitive: for fixed $\alpha$, with the reward range
and gradient scale controlled, weak token--execution dependence makes
the value component small. The bound says nothing about the size of
the residual.

The comparison holds at the pre-update parameters; it makes no claim
about the outcome of a finite optimization step. The reference parameter
$\alpha$ changes the decomposition but not the actual gradient $g_h$;
the residual need not be orthogonal to value ascent and is not a
uniquely identified causal component. Likewise, the fixed-prefix
derivative does not differentiate either the continuation policy or
the probability of visiting the recorded prefix. It is therefore not,
in general, a full-policy execution gradient or a full-sequence KL
gradient. Complete-advice insensitivity implies the token-level
condition only when the advice family covers every supported
completion, rather than a few selected strings.

For advisor learning, this means that an update with no local
execution-value component can still affect advice elsewhere through
shared parameters, and the lemma alone does not say whether that effect
helps or hurts. Appendix~\ref{arc:app:main-specialization} examines
repeated learning under explicit teacher assumptions and shows how
persistent, weaker corrections can limit performance, which motivates
targeted supervision. Neither result, however, claims that every
insensitive correction is harmful, and neither identifies \method's
predictive score with $I_h$.

\section{Retained corrections and the learning limit}
\label{arc:app:dynamics}

The local analysis in Appendix~\ref{arc:app:channel} leaves open what
happens when the advisor repeatedly learns from feedback. An update
changes its advice, which changes the failures encountered on later
rollouts and hence the corrections available for training. This
appendix makes that feedback loop explicit. We first prove
Theorem~\ref{arc:thm:limit} for the two-teacher model, then add reward
learning to establish Theorem~\ref{arc:thm:combined} and
Corollary~\ref{arc:cor:exposure}. Finally, we allow stationary random
teachers and action-dependent selection to identify which parts of the
argument depend on having only one teacher per type.

\subsection{Proof of Theorem~\ref{arc:thm:limit} in the two-teacher model}
\label{arc:app:main-specialization}

Theorem~\ref{arc:thm:limit} concerns a shared preference used in two
kinds of situation. Advice can improve execution in one kind, but not
the other. The proof will show how the mixture of retained corrections
sets the target of this shared preference. In particular, we must
account for the actual episode loss: averaging over a random number
of retained corrections is not the same as replacing that number by
its expectation.

Let $\beta\in(0,1)$ be the probability of an insensitive situation.
Both advice choices induce the same execution law there, with success
probability $s_I\in(0,1)$. In sensitive situations, choices 1 and 2
succeed with probabilities $s_1$ and $s_2$, where
$0<s_2<s_1<1$. The advisor chooses advice 1 with probability
$p(\theta)=1/(1+e^{-\theta})$, using the same log-odds
$\theta\in\mathbb R$ in every situation. Its expected success is
therefore
\[
J(\theta)=\beta s_I+(1-\beta)
\bigl[p(\theta)s_1+(1-p(\theta))s_2\bigr],
\qquad
J'(\theta)=(1-\beta)(s_1-s_2)p(1-p)>0.
\]
Increasing the shared preference for advice 1 thus improves success,
even though it has no effect within insensitive situations.

An insensitive failure supplies teacher $Q_I$, and a sensitive failure
supplies $Q_S$. These distributions are fixed and positive on both
choices. Write $\ell_j=\log[Q_j(1)/Q_j(2)]$ for their target log-odds
and assume $\ell_I<\ell_S$. This ordering means that the teacher from
an insensitive failure gives weaker support to advice 1. This is an
assumption about the feedback; equal execution laws do not imply it.
For example, the neutral teacher $Q_I=(0.5,0.5)$ and the more decisive
$Q_S=(0.8,0.2)$ satisfy it, and neither favors advice 2.

Each episode contains $M$ independent draws of situation, advice, and
outcome from this model, where $M$ is a fixed positive integer. If $N$
failures occur, a uniformly sampled subset of size $\min(N,b)$ is
proposed, where $b\ge1$ is a fixed integer cap. Each proposal is then retained
independently with probability $r_I$ or $r_S$, according to its type,
with $r_I,r_S\in(0,1]$. The episode loss averages reverse KL over the
retained corrections and is zero when none remain. As in the main
text, the sampled data and selections are held fixed when computing
the gradient; expectation is taken only afterwards.

For the calculation, write
\[
f_I=1-s_I,\qquad f_1=1-s_1,\qquad f_2=1-s_2,\qquad
F(p)=pf_1+(1-p)f_2.
\]
Here $F(p)$ is the sensitive failure probability, and
$0<f_1<f_2<1$. We also use
$A=\beta f_Ir_I$ and $V(p)=(1-\beta)F(p)r_S$ for the probabilities
that an individual opportunity is, respectively, an insensitive or
sensitive failure that would pass retention. These quantities do not
yet include the proposal cap.

\begin{proof}[Proof of Theorem~\ref{arc:thm:limit}]
For a fixed teacher $Q=(Q(1),Q(2))$ with log-odds $\ell$, the
student's reverse KL is
\[
\KL((p,1-p)\|Q)
=p\log\frac{p}{Q(1)}+(1-p)\log\frac{1-p}{Q(2)}.
\]
Using $dp/d\theta=p(1-p)$ and $\log[p/(1-p)]=\theta$, its derivative
is
\begin{equation}
\frac{d}{d\theta}\KL((p,1-p)\|Q)
=p(1-p)\left[\log\frac{p}{1-p}-\log\frac{Q(1)}{Q(2)}\right]
=p(1-p)(\theta-\ell).
\label{arc:eq:spec-onepull}
\end{equation}
Thus each correction pulls the shared log-odds toward its teacher's
target. The remaining question is which targets appear in the
episode's retained average.

Each opportunity fails independently with probability
$f=\beta f_I+(1-\beta)F(p)$, so $N\sim\operatorname{Binomial}(M,f)$.
Conditional on failure, its type is insensitive with probability
$\beta f_I/f$. Conditional on $N=n$, the types of the $n$ failures
are independent draws from this conditional type law. Sampling
$\min(n,b)$ of their indices uniformly, without using their types,
preserves that law for the proposed corrections. This statement
averages over the failure types; it does not assert independence
after an entire finite pool of typed failures has been fixed.

A proposed correction is retained and insensitive with probability
$q_I=\beta f_Ir_I/f=A/f$, retained and sensitive with probability
$q_S=(1-\beta)F(p)r_S/f=V(p)/f$, and otherwise dropped.
These outcomes are independent across the proposed positions,
conditional on $N=n$. Conditioning on the positions retained leaves
their types independent with the renormalized probabilities
$q_I/(q_I+q_S)$ and $q_S/(q_I+q_S)$. In particular, for every positive
retained count $K=k$, the probability that a retained correction is
insensitive and the expected average teacher log-odds are
\begin{equation}
\omega_I(\theta)=\frac{A}{A+V(p)},\qquad
\mu(\theta)=\omega_I(\theta)\ell_I+
\bigl[1-\omega_I(\theta)\bigr]\ell_S.
\label{arc:eq:spec-retained-mixture}
\end{equation}
Neither expression depends on the failure count or retained count,
so the same conditional mean holds after averaging over those counts.

For a realized episode with $K>0$, differentiating its loss while
holding the sample fixed gives
\[
\nabla_\theta L_{\rm ep}
=p(1-p)\left[\theta-\frac1K
\sum_{k\ \mathrm{retained}}\ell_{j(k)}\right],
\]
where $j(k)$ is the correction's type. The derivative is zero for
$K=0$. Conditional averaging using
Eq.~\eqref{arc:eq:spec-retained-mixture} therefore yields
\[
g(\theta)=\mathbb E[\nabla_\theta L_{\rm ep}]
=\Pr(K>0)\,p(1-p)[\theta-\mu(\theta)]
=B(\theta)p(1-p)[\theta-\mu(\theta)].
\]
This proves Eq.~\eqref{arc:eq:gradient} with the actual random
denominator of the episode loss. It does not differentiate the
sampling distribution: the dependence of $B$ and $\mu$ on $\theta$
describes how the expected sampled gradient changes between updates,
not additional derivative terms within an update.

For completeness, let $u=q_I+q_S=(A+V(p))/f$ be a proposal's
retention probability. Conditional on $N=n$, the probability of
retaining at least one of the $\min(n,b)$ proposals is
$1-(1-u)^{\min(n,b)}$. Hence
\begin{equation}
B(\theta)=\sum_{n=1}^M\binom Mn f^n(1-f)^{M-n}
\left[1-(1-u)^{\min(n,b)}\right]>0.
\label{arc:eq:spec-exposure}
\end{equation}
The cap and episode size change this exposure, but do not change the
conditional mean target in Eq.~\eqref{arc:eq:spec-retained-mixture}.

We next determine how that target changes as advice improves.
Since $F'(p)=f_1-f_2<0$, the sensitive retained-failure weight $V(p)$
decreases with $p$, whereas the insensitive weight $A$ stays fixed.
Consequently,
\[
\frac{d\omega_I}{dp}
=\frac{A(1-\beta)r_S(f_2-f_1)}{[A+V(p)]^2}>0,
\qquad
\mu'(\theta)=-(\ell_S-\ell_I)\omega_I'(\theta)<0.
\]
The learning target therefore falls as advice improves: sensitive
failures become rarer, and the weaker insensitive teacher supplies a
larger share of the retained supervision.

Let $H(\theta)=\theta-\mu(\theta)$. Its derivative satisfies
$H'(\theta)>1$, while $\mu(\theta)$ remains strictly between
$\ell_I$ and $\ell_S$. It follows that $H$ tends to opposite infinities
at the two ends of the real line and has a unique zero
$\theta^*=\mu(\theta^*)\in(\ell_I,\ell_S)$.
The distillation flow is
$\dot\theta=-B(\theta)p(1-p)H(\theta)$.
Its field is smooth and bounded: $B\le1$, the target is bounded,
and $p(1-p)|\theta|$ is bounded. Solutions therefore exist uniquely
for all time. Since $B(\theta)p(1-p)>0$ at every finite $\theta$,
the field points upward below $\theta^*$ and downward above it.
By uniqueness, a solution cannot cross this equilibrium. It is thus
monotone and bounded, and its limit must be the sole zero of the field,
$\theta^*$. This proves convergence from every finite initialization.

Finally, divide the numerator and denominator of $\omega_I$ by
$r_S$ and set $\rho=r_I/r_S$. Then
\[
\omega_I(\theta;\rho)
=\frac{\rho\beta f_I}{\rho\beta f_I+(1-\beta)F(p)},
\qquad
\partial_\rho\omega_I
=\frac{\beta f_I(1-\beta)F(p)}
{[\rho\beta f_I+(1-\beta)F(p)]^2}>0.
\]
Thus $\partial_\rho\mu<0$: reducing the relative retention of
insensitive corrections raises the mean target at every fixed
preference. Implicitly differentiating the equilibrium equation gives
\begin{equation}
\frac{d\theta^*}{d\rho}
=\frac{\partial_\rho\mu(\theta^*;\rho)}
{1-\partial_\theta\mu(\theta^*;\rho)}<0.
\label{arc:eq:spec-retention-shift}
\end{equation}
The denominator is positive by the monotonicity just established,
and $J'(\theta)>0$, so the limiting success also increases as
$\rho$ decreases. The equilibrium equation contains neither $M$ nor
$b$, and depends on retention only through $\rho$. Changing episode
size or proposal cap, or multiplying both retention probabilities by
the same admissible positive factor, therefore leaves the limit
unchanged. This completes the proof.
\end{proof}

The result distinguishes the composition of supervision from its
frequency. Selective retention changes the target $\mu$ and hence the
equilibrium. Uniform thinning changes $B$ but leaves that target
unchanged. It can change how quickly the flow moves, though not
necessarily by a constant rescaling of time, because $B$ depends
nonlinearly on the retention probabilities. This distinction is why we
compare \method with a count-matched control as well as with learning
from every proposal. With reward learning, exposure also affects the
limit, as the next subsection shows.

\paragraph{An illustration with neutral and informative teachers.}
Take $\beta=s_I=1/2$, $s_1=0.9$, $s_2=0.1$,
$Q_I=(0.5,0.5)$, and $Q_S=(0.8,0.2)$. The insensitive teacher is
neutral, whereas the sensitive teacher favors the more successful
advice. Solving $\theta=\mu(\theta)$ gives the following values,
rounded to three decimals:
\begin{center}
\begin{tabular}{@{}lrrrr@{}}
\toprule
Retention ratio $\rho$ & $\theta^*$ & $p(\theta^*)$
& $\omega_I(\theta^*)$ & $J(\theta^*)$\\
\midrule
$1$ & $0.601$ & $0.646$ & $0.566$ & $0.558$\\
$1/4$ & $0.994$ & $0.730$ & $0.283$ & $0.592$\\
Limit as $\rho\downarrow0$ & $\log 4$ & $0.800$ & $0$ & $0.620$\\
\bottomrule
\end{tabular}
\end{center}
With equal retention, insensitive corrections make up more than half
of supervision at the equilibrium, even though they come from half of
the situations. Retaining them at one quarter of the sensitive rate
reduces their share and raises the learned preference. As their
relative retention tends to zero, the advisor approaches the sensitive
teacher's own preference of $0.8$. The last row is a limiting case; the
theorem itself still assumes strictly positive retention.
For comparison, $J=0.5$ at $p=1/2$ and $J=0.7$ at $p=1$.
Selection improves the distillation equilibrium in this example, but a
teacher of finite strength still does not lead the advisor to the
success maximizer.

\paragraph{What changes when the teacher changes between updates.}
\label{arc:app:self-teachers}
The preceding theorem fixes teacher targets over repeated learning.
This is different from detaching a teacher for one update.
Lemma~\ref{arc:lem:channel} only requires the latter: its identity can
be applied anew at each snapshot, even when the feedback-conditioned
teacher changes between snapshots. A statement about the limiting
preference, however, must also describe how those targets evolve.
Two simple choices show why the distinction matters.

Keep the other assumptions of Appendix~\ref{arc:app:main-specialization}
and let the detached target for type $j$ at snapshot $\bar\theta$ have
log-odds
\[
\ell_j(\bar\theta)=\kappa\bar\theta+(1-\kappa)\ell_j,
\qquad 0\le\kappa<1,
\]
with the same fixed $\kappa$ for both types and fixed
$\ell_I<\ell_S$. Each teacher now moves partway with the student
while retaining a pull toward its original target. Because the target
is detached during differentiation,
Eq.~\eqref{arc:eq:spec-onepull} uses
$\bar\theta-\ell_j(\bar\theta)=(1-\kappa)(\bar\theta-\ell_j)$.
Averaging over the unchanged proposal and retention process gives
\[
g_\kappa(\bar\theta)
=(1-\kappa)B(\bar\theta)p(1-p)
[\bar\theta-\mu(\bar\theta)]
=(1-\kappa)g(\bar\theta).
\]
The distillation-only flow consequently has the same equilibrium and
the same trajectories after a constant rescaling of time. When reward
learning is also present, this rescaling applies only to the
distillation term: its effective weight becomes
$\lambda(1-\kappa)$. The combined equilibrium can therefore change,
although Theorem~\ref{arc:thm:combined} still compares selectors at
this common positive effective weight.

Now consider detached targets that instead stay a fixed, nonnegative
distance ahead of the student's log-odds,
\[
\ell_j(\bar\theta)=\bar\theta+\delta_j,
\qquad \delta_I=0<\delta_S.
\]
The insensitive teacher then gives zero distillation gradient, while
the sensitive teacher contributes $-p(1-p)\delta_S$. The resulting
distillation flow is
\[
\dot\theta=B(\theta)p(1-p)
\bigl[1-\omega_I(\theta)\bigr]\delta_S>0
\]
at every finite $\theta$. Its field is smooth and bounded, so the
solution exists for all time and increases. It cannot have a finite
limit, because the field is strictly positive at any such limit.
Hence $\theta\to\infty$ and $p\to1$; the flow does not settle at a
finite equilibrium as in the fixed-teacher model.

We do not assume that either example describes how \method's teacher
actually evolves. Together, they show what detachment provides and what
it does not: it justifies the conditional gradient calculation, but it
does not by itself preserve a learning-limit theorem. The stationary
random-teacher extension in Appendix~\ref{arc:app:dynamics-general}
allows variation in feedback while making the needed stability of its
conditional law explicit.

\subsection{Proof of Theorem~\ref{arc:thm:combined} and Corollary~\ref{arc:cor:exposure} in the two-teacher model}
\label{arc:app:combined-two-teacher}

Adding reward learning changes the comparison between selectors because
the amount of distillation now matters as well as its mean target. We
prove that preferential retention still gives a higher limiting success
rate than retaining every proposal, and then consider uniform thinning,
which changes exposure without changing the retained mixture. Throughout,
we use the two-teacher model of Appendix~\ref{arc:app:main-specialization},
with common fixed weights $c_0\ge0$ and $\lambda>0$. Both selectors start
from the same finite logit. The reward term is exact ascent on expected
success, as in Section~\ref{sec:properties}.

Write $p=p(\theta)$, $f_I=1-s_I$, $f_1=1-s_1$, and $f_2=1-s_2$.
The sensitive failure probability is $F(p)=pf_1+(1-p)f_2$, and the
overall failure probability is $f=\beta f_I+(1-\beta)F(p)$. For a
selector $S=(r_I^S,r_S^S)$, define
\[
\begin{aligned}
A_S&=\beta f_Ir_I^S,
&V_S(p)&=(1-\beta)F(p)r_S^S,\\
\omega_S(\theta)&=\frac{A_S}{A_S+V_S(p)},
&\mu_S(\theta)&=\ell_S-(\ell_S-\ell_I)\omega_S(\theta).
\end{aligned}
\]
Thus $\mu_S$ is the mean retained teacher log-odds. A proposed failure
is retained with probability $u_S=(A_S+V_S(p))/f$. Since an episode
contains $M$ independent opportunities and proposes $\min(N,b)$ of its
$N$ failures, its probability of receiving any distillation is
\[
B_S(\theta)=\sum_{n=1}^M\binom Mn f^n(1-f)^{M-n}
\bigl[1-(1-u_S)^{\min(n,b)}\bigr].
\]
The expected sampled-loss gradient from Theorem~\ref{arc:thm:limit}
is $g_S=B_Sp(1-p)(\theta-\mu_S)$. Using
$J'(\theta)=(1-\beta)(s_1-s_2)p(1-p)$, we can therefore write the
combined update field as
\begin{equation}
F_S(\theta)=p(1-p)\bigl\{C-\lambda B_S(\theta)
   [\theta-\mu_S(\theta)]\bigr\},
\qquad C=c_0(1-\beta)(s_1-s_2)\ge0.
\label{arc:eq:tt-field}
\end{equation}
Here $S=0$ means $r_I^0=r_S^0=1$, whereas the preferential selector
$G$ has fixed independent retention rates in $(0,1]$ with
$r_I^G/r_S^G<1$.

\begin{proof}[Proof of Theorem~\ref{arc:thm:combined}]
We first compare the two selectors at a common parameter $\theta$.
Dividing the numerator and denominator of $\omega_S$ by $r_S^S$
shows that it depends on the retention rates only through their
ratio $\rho_S=r_I^S/r_S^S$:
\[
\omega_S(\theta)=
\frac{\rho_S\beta f_I}
     {\rho_S\beta f_I+(1-\beta)F(p)}.
\]
This expression is strictly increasing in $\rho_S$. Since
$\rho_G<\rho_0=1$ and $\ell_I<\ell_S$, preferential retention gives
$\mu_G(\theta)>\mu_0(\theta)$. It also gives $u_G\le u_0$, because
neither type is retained more often than under no gating. Each term
$1-(1-u_S)^{\min(n,b)}$ is increasing in $u_S$, so $B_G\le B_0$.
In particular, for every finite $\theta$,
\begin{equation}
\mu_G(\theta)>\mu_0(\theta),\qquad
0<B_G(\theta)\le B_0(\theta).
\label{arc:eq:tt-order}
\end{equation}

For later use, the exposure probability is bounded away from zero
uniformly in $\theta$. Indeed, whenever $N\ge1$, the positive cap
ensures that at least one failure is proposed. Retaining at least one
of these proposals has probability at least $u_S$. As $M\ge1$,
$\Pr(N\ge1)=1-(1-f)^M\ge f$, and consequently
\begin{equation}
B_S(\theta)\ge u_S\Pr(N\ge1)
\ge fu_S=A_S+V_S(p)\ge A_S>0.
\label{arc:eq:tt-exposure-bound}
\end{equation}
This lower bound follows from the positive probability of an
insensitive failure and its fixed positive retention rate.

Subtracting the two instances of Eq.~\eqref{arc:eq:tt-field} gives
\[
\begin{aligned}
F_G-F_0
&=\lambda p(1-p)
  \bigl[B_0(\theta-\mu_0)-B_G(\theta-\mu_G)\bigr]\\
&=\lambda p(1-p)
  \bigl[B_G(\mu_G-\mu_0)+(B_0-B_G)(\theta-\mu_0)\bigr].
\end{aligned}
\]
This is the composition--exposure identity in
Eq.~\eqref{arc:eq:joint-update}. Its composition term is strictly
positive, while the sign of its exposure term depends on whether the
current preference lies above or below the ungated teaching target.
We will use this sign change to distinguish the limiting comparison
from a comparison at every finite time.

Before comparing the trajectories, we establish that each one has a
finite limit. The field $F_S$ is smooth because all denominators in
the expressions above are positive. It is also bounded on
$\mathbb R$: $B_S\le1$, $\mu_S\in(\ell_I,\ell_S)$, and both
$p(1-p)$ and $|\theta|p(1-p)$ are bounded. The positive lower bound
in Eq.~\eqref{arc:eq:tt-exposure-bound} further gives the restoring
signs
\begin{equation}
\begin{aligned}
F_S(\theta)&>0 &&\text{if }\theta<\ell_I,\\
F_S(\theta)&<0 &&\text{if }\theta>
  \ell_S+\frac{C}{\lambda A_S}.
\end{aligned}
\label{arc:eq:tt-restoring}
\end{equation}
For the first inequality, $\theta-\mu_S<0$ makes the bracket in
Eq.~\eqref{arc:eq:tt-field} positive. For the second,
$B_S(\theta-\mu_S)\ge A_S(\theta-\ell_S)>C/\lambda$ makes it
negative. Thus we can enclose any finite initial value in a compact
interval on whose endpoints the field points inward. Smoothness gives
a unique solution, and the inward signs keep it in that interval for
all time. A nonconstant scalar autonomous trajectory cannot cross a
zero of its field: uniqueness would otherwise be violated by the
constant solution at that zero. Its direction therefore cannot
reverse, so it is monotone and bounded. It has a finite limit, and
continuity forces that limit to be a zero of $F_S$. A trajectory
started at a zero is constant and has the same conclusion. This
argument does not require the combined field to have a unique zero.

Let $a_0$ be the ungated distillation-only equilibrium. By
Theorem~\ref{arc:thm:limit}, $\theta-\mu_0(\theta)$ is strictly
increasing and vanishes at $a_0$. If $\theta<a_0$, then
$\theta<\mu_0(\theta)<\mu_G(\theta)$, so both combined fields are
positive. If $\theta\ge a_0$, the exposure term in
Eq.~\eqref{arc:eq:joint-update} is nonnegative and its composition
term is positive. We have therefore established
\begin{equation}
\begin{aligned}
F_0(\theta),\,F_G(\theta)&>0 &&(\theta<a_0),\\
F_G(\theta)&>F_0(\theta) &&(\theta\ge a_0),\\
F_0(a_0)&=p(a_0)[1-p(a_0)]C\ge0.
\end{aligned}
\label{arc:eq:tt-sign-comparison}
\end{equation}
In particular, $F_G(a_0)>0$. The half-line $[a_0,\infty)$ is
forward invariant for both flows: the gated field points inward at
its boundary, and the ungated field either points inward or has an
equilibrium there.

Now suppose the common initial value $x$ satisfies $x\ge a_0$.
Set $D(t)=\theta_G(t)-\theta_0(t)$. Initially $D(0)=0$ and
$D'(0)=F_G(x)-F_0(x)>0$, so the gated trajectory is strictly larger
for all sufficiently small positive times. If the trajectories first
met again at a time $t_1>0$, then $D'(t_1)\le0$. At that meeting
point, however, both parameters belong to $[a_0,\infty)$ and
Eq.~\eqref{arc:eq:tt-sign-comparison} gives
$D'(t_1)=F_G(\theta_G(t_1))-F_0(\theta_G(t_1))>0$, a contradiction.
Hence $\theta_G(t)>\theta_0(t)$ for every $t>0$. Their finite limits
are at least weakly ordered. Equality of these limits is impossible:
a common limit $z\ge a_0$ would be a zero of both fields, whereas
$F_G(z)>F_0(z)$. This proves strict limiting order as well as the
claimed finite-time comparison for starts at or above $a_0$.

It remains to compare limits when $x<a_0$, where exposure can prevent
such a finite-time ordering. If $c_0=0$, the ungated flow converges
to $a_0$ by Theorem~\ref{arc:thm:limit}. The gated field is positive
throughout $(-\infty,a_0]$, so its finite limiting equilibrium must
lie strictly above $a_0$. If $c_0>0$, then both fields are positive
on $(-\infty,a_0]$, including $a_0$ itself. Let $z_0$ be the first
zero of $F_0$ above $x$. Such a zero exists because the field is
positive at $x$ and negative sufficiently far to the right, and
$z_0>a_0$ because both fields are positive through $a_0$. The ungated
trajectory increases to $z_0$. On
$(a_0,z_0)$, its field is positive, so
$F_G>F_0>0$ there; at $z_0$, the strict field comparison gives
$F_G(z_0)>0$. Together with positivity below $a_0$, this shows that
$F_G$ has no zero anywhere from $x$ through $z_0$. Its finite
limiting equilibrium must therefore lie strictly above $z_0$.
In both cases $\theta_G^\infty>\theta_0^\infty$. Finally,
$J'(\theta)>0$ for every finite $\theta$, giving
$J(\theta_G^\infty)>J(\theta_0^\infty)$ as claimed.
\end{proof}

The same identity explains why a gain over no gating need not come
from a better correction mixture. Retaining every proposal with the
same probability leaves the conditional mean target unchanged, but
reduces how often that target is applied. This gives the following
comparison under the same common reward objective and initialization.

\begin{arccorollary}[Uniform thinning changes exposure, not composition]
\label{arc:cor:exposure}
If every proposal is retained independently with the same fixed
probability $r\in(0,1]$, the limiting preference is at least that
of no gating under the common reward objective. Without reward
learning ($c_0=0$), the two limits coincide.
\end{arccorollary}

\begin{proof}
Taking $r_I^G=r_S^G=r$ leaves the retention ratio equal to one,
so $\mu_G\equiv\mu_0$, while $0<B_G\le B_0$. Moreover,
Eq.~\eqref{arc:eq:tt-exposure-bound} gives
$B_G\ge rf\ge r\beta(1-s_I)>0$. Thus both fields have the restoring
signs and finite limiting equilibria established in the preceding
proof. Their difference now reduces to
\begin{equation}
F_G(\theta)-F_0(\theta)
=\lambda p(1-p)(B_0-B_G)[\theta-\mu_0(\theta)].
\label{arc:eq:tt-uniform-thinning}
\end{equation}
When $c_0=0$, the two fields have the same sign as
$\mu_0(\theta)-\theta$ and vanish only at $a_0$. Both flows therefore
converge to $a_0$ from every common finite initialization.

Suppose instead that $c_0>0$. On $[a_0,\infty)$,
Eq.~\eqref{arc:eq:tt-uniform-thinning} gives $F_G\ge F_0$, and
$F_G(a_0)=F_0(a_0)=p(a_0)[1-p(a_0)]C>0$. This half-line is
forward invariant. The scalar comparison principle for locally
Lipschitz fields then gives $\theta_G(t)\ge\theta_0(t)$ from a
common initial value in this half-line, hence the same weak ordering
of their limits.

For a common initial value below $a_0$, both fields are positive up
to and including $a_0$. As in the preceding proof, the ungated
trajectory increases to its first equilibrium $z_0>a_0$. On
$(a_0,z_0)$ we have $F_G\ge F_0>0$, so the gated field has no zero
between the common start and $z_0$. At $z_0$ it satisfies only
$F_G(z_0)\ge0$, which allows its limiting equilibrium to equal
$z_0$. In all cases its limit is at least $z_0$, proving the weak
comparison. The proof does not require strict inequality at $z_0$;
in particular, $r=1$ gives identical flows.
\end{proof}

Appendix~\ref{arc:app:dynamics-general} extends these arguments to
stationary random teachers, including conditions for comparisons
when selection changes the retained targets. The independent thinning
in Corollary~\ref{arc:cor:exposure} differs from the implemented
matched-count random control, whose gate-derived per-episode quotas can
change composition (Appendix~\ref{app:selection-controls}).

For a concrete illustration, use the instance from
Appendix~\ref{arc:app:main-specialization} with
$\beta=s_I=1/2$, $s_1=0.9$, $s_2=0.1$, $Q_I=(0.5,0.5)$, and
$Q_S=(0.8,0.2)$, so $a_0=0.601$ and $p(a_0)=0.646$. Set
$M=4$, $b=2$, $c_0=1/2$, $\lambda=1$, and take the quarter-rate
gate $G=(1/4,1)$. To separate composition from exposure in this
example, introduce the hypothetical field
\[
F_R(\theta)=c_0J'(\theta)
-\lambda B_G(\theta)p(1-p)[\theta-\mu_0(\theta)].
\]
By construction, it combines the ungated composition with the gated
exposure; it is not the same-episode matched-count control. From
$\theta(0)=0$,
the three fields converge to
\begin{center}
\small
\begin{tabular}{@{}lrrr@{}}
\toprule
Dynamics & $\theta^\infty$ & $p^\infty$ & $J(\theta^\infty)$\\
\midrule
No gating ($F_0$)
& 0.794 & 0.689 & 0.576\\
Ungated composition, gated exposure ($F_R$)
& 0.876 & 0.706 & 0.582\\
Preferential selection ($F_G$)
& 1.303 & 0.786 & 0.615\\
\bottomrule
\end{tabular}
\end{center}
At the ungated limit, the composition and exposure terms on the
right-hand side of Eq.~\eqref{arc:eq:joint-update} are $0.287$ and
$0.057$, respectively. Thus changing exposure alone accounts for
part of the limiting gain in this instance, while changing the
retained target raises the limit further.

The initial effect can go the other way. Keep the same success laws,
teachers, weights, and gate, but take one opportunity per episode
($M=b=1$). At a common start of $\theta=-2$,
$F_G(-2)=0.177<F_0(-2)=0.217$, so the gated advisor starts
more slowly, although the limiting advice-1 probabilities remain
ordered: $p^\infty=0.736$, $0.824$, and $0.892$ for $F_0$, $F_R$,
and $F_G$, respectively. Below $a_0$, removing corrections can
weaken an upward teaching pull; the comparison of limits does not say
that every earlier update also improves.

The finite balance studied here depends on the common positive
distillation weight. With $c_0>0$ and no distillation, reward-only
ascent has $\dot\theta=c_0J'(\theta)>0$ at every finite logit and
approaches $p=1$, the best success rate in this model. The theorem
therefore compares preferential selection with ungated distillation.
\method's gains over outcome-only GRPO under finite training schedules
are established by the experiments.

\subsection{Stationary random teachers and action-dependent selection}
\label{arc:app:dynamics-general}\label{arc:app:combined}

The two-teacher model isolates how the mixture of retained corrections
affects learning by assigning a fixed target to each situation type.
In a feedback-based system, however, the target may vary even within
one type: the feedback can depend on the advice that was issued, and
the selector can preferentially retain some of the resulting targets.
This subsection allows both forms of variation while keeping the
two-action policy and the opportunity and execution laws fixed. We
first identify the teacher statistic that determines the expected
distillation update, then establish conditions for a unique learning
limit and for comparing selectors with a common reward objective.

The relevant stationarity requirement concerns the joint law of the
teacher and its retention, conditional on the type, action, and
failure. It does not require retention to be independent of teacher
content. As the student changes its action probabilities, the mixture
of these conditional laws can still change, and this is why the target
below depends on the parameter.

\begin{arcassumption}[Stationary random feedback and selection]
\label{arc:ass:dynamics}
Keep the opportunity and execution laws, independent sampling, uniform
capped failure proposals, and episode-averaged loss of
Appendix~\ref{arc:app:main-specialization}. Conditional on type
$j\in\{I,S\}$, issued action $a\in\{1,2\}$, and failure, the
teacher $Q$ and retention indicator have a joint law independent of
$\theta$; these pairs are independent across opportunities, and the
uniform proposal sample is independent of them. Teachers
are positive on both actions, retention has probability
$r_{ja}\in(0,1]$, and the retained teacher log-odds have a finite
absolute mean. Teachers and selections are detached when differentiating.
\end{arcassumption}

For each type and issued action, define the mean teacher log-odds
among the corrections that survive selection:
\[
\ell_{ja}=\mathbb E\!\left[
\log\frac{Q(1)}{Q(2)}\,\middle|\,j,a,\text{failure, retained}\right].
\]
Conditioning on retention matters when the selector depends on the
teacher: it records the target actually supplied to the loss, which
can differ from the mean target before selection. Note also that
$\ell_{ja}$ averages log-odds rather than taking the log-odds of the
mean teacher probabilities, because log-odds enter the reverse-KL
gradient linearly in Eq.~\eqref{arc:eq:spec-onepull}.

Using $f_I=1-s_I$ and $f_a=1-s_a$ from
Appendix~\ref{arc:app:main-specialization}, define
\begin{align}
i_a&=\beta f_Ir_{Ia},\qquad v_a=(1-\beta)f_ar_{Sa},\qquad C_a=i_a+v_a,
\nonumber\\
\bar\ell_a&=\frac{i_a\ell_{Ia}+v_a\ell_{Sa}}{C_a},\qquad
\mu(\theta)=\frac{pC_1\bar\ell_1+(1-p)C_2\bar\ell_2}
{pC_1+(1-p)C_2},\quad p=\operatorname{sigmoid}(\theta).
\label{arc:eq:mixture}
\end{align}
Conditional on issuing action $a$, the quantities $i_a$ and $v_a$
are the probabilities of failures of each type whose corrections would
be retained if proposed. Their sum
$C_a$ is positive, and $\bar\ell_a$ is the corresponding retained
mean target. Averaging across the two actions uses the weights $pC_1$
and $(1-p)C_2$ instead of the action probabilities $p$ and $1-p$,
because the actions can produce retained failures at different rates.
This gives $\mu(\theta)$, the mean log-odds of a
retained proposal at the current policy. Let $B(\theta)$ again denote
the probability of retaining at least one correction in an episode.

\begin{arctheorem}[Stationary random-teacher extension]
\label{arc:thm:limit-formal}
Under Assumption~\ref{arc:ass:dynamics}, the expected sampled-loss
gradient is $g=Bp(1-p)(\theta-\mu)$, with $B>0$.
If $\bar\ell_1-\bar\ell_2\le4$, its gradient flow converges from every
finite initialization to the unique $\theta^*=\mu(\theta^*)$.
The root lies strictly between unequal $\bar\ell_1,\bar\ell_2$, or
equals their common value, and is independent of episode size and cap.
For classwise retention $r_{ja}=r_j$, suppose additionally that
$\ell_{j2}\ge\ell_{j1}$ for each type and
$\max_a\ell_{Ia}<\min_a\ell_{Sa}$. Holding these four retained means
fixed, decreasing $r_I/r_S$ strictly increases $\theta^*$ and its
expected success.
\end{arctheorem}

\begin{proof}
We begin by taking the expectation of the realized episode gradient.
For any retained teacher, Eq.~\eqref{arc:eq:spec-onepull} gives the
contribution $p(1-p)(\theta-\operatorname{logit}Q(1))$. Write
\[
f=\beta f_I+(1-\beta)[pf_1+(1-p)f_2],\qquad
t=pC_1+(1-p)C_2,\qquad u=t/f.
\]
Here $f$ is the probability that an opportunity fails, $t$ is the
probability of a failure whose correction would be retained, and $u$
is the retention probability conditional on failure. The conditional
mean log-odds of a retained correction is $\mu$ by
Eq.~\eqref{arc:eq:mixture}.

Conditional on the number of failures, uniform proposal sampling
does not favor any type, action, or teacher--retention pair. The
proposed pairs therefore have the same independent conditional law
as failure opportunities. Conditioning further on the retained
indices leaves their teachers distributed according to the retained
law. In particular, for every positive realized retained count $K$,
the expected average of their log-odds is $\mu$. The episode loss
averages the $K$ realized contributions and is zero when $K=0$.
Consequently, exactly as in the derivation of
Eq.~\eqref{arc:eq:gradient},
\[
g=\Pr(K\ge1)\,p(1-p)(\theta-\mu)
=Bp(1-p)(\theta-\mu).
\]
The exposure probability $B$ is the capped binomial expression in
Eq.~\eqref{arc:eq:spec-exposure}, with the $f,u$ defined here. Its
positivity follows from positive failure and retention probabilities.
Finite absolute log moments make each of these expectations finite.
This calculation averages the loss with its realized denominator
$K$; it does not replace that denominator by an expected count.

To analyze the limit, we must account for the dependence of $\mu$ on
the policy. Unlike in the two-teacher model, increasing the probability
of action 1 can increase the retained mean target. Let
$D=\bar\ell_1-\bar\ell_2$ and
$\omega=\operatorname{sigmoid}(\theta+\log(C_1/C_2))$. Since
$p/(1-p)=e^\theta$, the retained share of action 1 is
\[
\frac{pC_1}{pC_1+(1-p)C_2}=\omega.
\]
Thus the target and the derivative of the fixed-point residual are
\[
\mu=\bar\ell_2+D\omega,\qquad
\frac{d}{d\theta}(\theta-\mu)=1-D\omega(1-\omega).
\]
The bound $\omega(1-\omega)\le1/4$ makes this derivative strictly
positive whenever $D<4$. At the boundary $D=4$, the derivative is
nonnegative and can vanish only where $\omega=1/2$, which occurs
at a single finite parameter. Integrating the derivative over any
nontrivial interval therefore still gives a positive difference.
Hence $\theta-\mu(\theta)$ is strictly increasing also at $D=4$.

Because $\mu$ is a convex combination of the two finite constants
$\bar\ell_1,\bar\ell_2$, it is bounded. The residual consequently
tends to $-\infty$ and $+\infty$ at the respective ends of the real
line, so it has exactly one zero $\theta^*$. Both action weights are
strictly positive at any finite parameter. Evaluating the convex
combination at $\theta^*=\mu(\theta^*)$ places this root strictly
between unequal targets, or at their common value if they coincide.

The descent flow is $\dot\theta=-g$. Its field is smooth and bounded:
$0<B\le1$, $\mu$ is bounded, and $p(1-p)|\theta|$ is bounded.
Furthermore, it is positive below $\theta^*$ and negative above
$\theta^*$. A trajectory from a finite initialization therefore stays
between its initial value and $\theta^*$ and cannot cross the
equilibrium, by uniqueness of solutions. It is monotone unless
already stationary, and its finite limit must be a zero of the field
by continuity. The only such zero is $\theta^*$, proving convergence.
Neither the episode size $M$ nor the cap $b$ enters the equation
$\theta^*=\mu(\theta^*)$; they affect the positive factor $B$ and
hence the speed of this flow, but not its limit.

For the retention comparison, assume $r_{ja}=r_j$. Conditional on
either issued action, the insensitive retained weight is then
$\beta f_Ir_I$, whereas the sensitive weight is
$(1-\beta)f_ar_S$. Since $f_1<f_2$, action 1 has a larger insensitive
share among its retained failures. The assumed ordering
$\ell_{j1}\le\ell_{j2}$ means that replacing each action-1 target
by its action-2 counterpart cannot lower the mixture. After that
replacement, reducing the insensitive share cannot lower it either,
because every insensitive target is strictly below every sensitive
target. These two comparisons give $\bar\ell_1\le\bar\ell_2$.
In particular, the fixed-point residual has derivative
$1-\partial_\theta\mu\ge1$, so the root is unique and varies
smoothly with the retention ratio.

To determine the direction of that change, put $\rho=r_I/r_S$,
$F=pf_1+(1-p)f_2$,
$L_I=p\ell_{I1}+(1-p)\ell_{I2}$, and
$L_S=[pf_1\ell_{S1}+(1-p)f_2\ell_{S2}]/F$. Dividing the numerator
and denominator in Eq.~\eqref{arc:eq:mixture} by $r_S$ gives
\[
\mu=
\frac{\rho\beta f_IL_I+(1-\beta)FL_S}
{\rho\beta f_I+(1-\beta)F}.
\]
When $\theta$ and the four retained means are fixed, $L_I,L_S,F$
are fixed as well. Differentiating this weighted average with respect
to $\rho$ yields
\[
\partial_\rho\mu=
\frac{\beta f_I(1-\beta)F(L_I-L_S)}
{[\rho\beta f_I+(1-\beta)F]^2}<0.
\]
The inequality follows because $L_I$ and $L_S$ are convex averages
within their respective types and the types are strictly separated.
Implicitly differentiating $\theta^*-\mu(\theta^*,\rho)=0$ now gives
\[
\frac{d\theta^*}{d\rho}
=\frac{\partial_\rho\mu}{1-\partial_\theta\mu}<0.
\]
Finally, $J'(\theta)=(1-\beta)(s_1-s_2)p(1-p)>0$, so the higher
limiting parameter also has strictly higher expected success.
\end{proof}

The fixed-point result and the retention comparison have different
requirements. The first result allows random teachers and
action-dependent selection, subject to the stated gap condition.
The comparison additionally holds the four retained means fixed while
varying $r_I/r_S$. This restriction matters because a new selector can
change teacher content within a type, and therefore change
$\ell_{ja}$ as well as $r_{ja}$. Retaining fewer insensitive
corrections alone does not imply an improved target. In the original
two-teacher model, the fixed-mean restriction holds by construction:
$\ell_{ja}=\ell_j$, and $\bar\ell_1\le\bar\ell_2$ makes the gap
condition automatic.

\paragraph{Reward learning when selection also changes the targets.}
To compare selectors that can change the retained teacher laws, we
formulate the required improvement directly in terms of their mean
targets. This extends the reward-learning comparison of
Appendix~\ref{arc:app:combined-two-teacher}: the proof needs an
ordering of $\mu_G$ and $\mu_0$, together with an ordering of their
supervision probabilities, rather than an ordering of classwise
retention rates alone.

\begin{arcassumption}[Comparable selective dynamics]
\label{arc:ass:comparison}
Systems $0,G$ satisfy Assumption~\ref{arc:ass:dynamics} with common
opportunity and execution laws, episode size, and proposal cap.
System $0$ keeps every proposal; $G$ keeps a subset. Their retained
teacher laws may differ. For every finite $\theta$, assume
$\mu_G(\theta)>\mu_0(\theta)$, and require
$\bar\ell^0_1-\bar\ell^0_2\le4$.
\end{arcassumption}

\begin{arctheorem}[Selection with stationary random teachers]
\label{arc:thm:combined-formal}
Under Assumption~\ref{arc:ass:comparison}, fix $c_0\ge0$, $\lambda>0$
and $F_j=c_0J'-\lambda g_j$. From every common finite initialization,
the flows $\dot\theta_j=F_j(\theta_j)$ converge to finite limits with
$\theta_G^\infty>\theta_0^\infty$ and
$J(\theta_G^\infty)>J(\theta_0^\infty)$, even if combined equilibria
are not unique. From a common start at or above the ungated
distillation-only equilibrium $a_0$, also
$\theta_G(t)>\theta_0(t)$ for every $t>0$.
\end{arctheorem}

\begin{proof}
Subset retention gives $0<B_G\le B_0$: on common proposals, an
episode with a retained gated correction necessarily has a retained
ungated correction. We also need a lower bound that holds uniformly
in $\theta$, so that the distillation term continues to provide a
restoring force at large parameter values. For either selector $j$,
when an episode contains at least one failure, the proposal cap
$b\ge1$ leaves at least one opportunity for retention. Thus the
capped binomial expression gives
\[
B_j\ge u_j\Pr(N\ge1)\ge u_jf
=pC_1^j+(1-p)C_2^j\ge\min_a C_a^j>0.
\]
Here $\Pr(N\ge1)=1-(1-f)^M\ge f$ because $M\ge1$.
The constants $C_a^j$ are positive by the failure and retention
assumptions.

The means $\mu_j$ are bounded, and each field has the form
Eq.~\eqref{arc:eq:tt-field}, with
$C=c_0(1-\beta)(s_1-s_2)$. The bracket in that equation is positive
for sufficiently negative $\theta$ and, by the uniform lower bound
on $B_j$, negative for sufficiently positive $\theta$. The fields
are smooth and bounded. Each flow therefore has a unique global
solution confined to a compact interval containing its start. A
scalar autonomous trajectory cannot cross an equilibrium; it is
otherwise monotone, and continuity forces its finite limit to be an
equilibrium. This is the same convergence argument used for the
two-teacher fields, and it does not require their equilibria to be
unique.

The baseline gap condition and
Theorem~\ref{arc:thm:limit-formal} imply that
$\theta-\mu_0(\theta)$ is strictly increasing with unique zero
$a_0$. Below this point, $\theta<\mu_0<\mu_G$, so both fields are
positive. At and above $a_0$, the composition--exposure identity
Eq.~\eqref{arc:eq:joint-update} has a strictly positive composition
term and a nonnegative exposure term, giving $F_G>F_0$. In
particular, $F_0(a_0)=p(1-p)C\ge0$ and $F_G(a_0)>F_0(a_0)$.
These are precisely the sign properties that supported the
two-teacher comparison, now obtained from assumptions on the
retained mean targets.

Suppose first that the common initialization is at or above $a_0$.
The half-line $[a_0,\infty)$ is forward invariant for both flows.
At a common point in this region, the gated field is strictly
larger, so the trajectory difference has positive derivative at
time zero. If the difference subsequently returned to zero for the
first time, its derivative there would be nonpositive. At that
common parameter, however, its derivative is $F_G-F_0>0$, a
contradiction. This proves strict ordering at every positive time
and weak ordering of the limits. The limits cannot be equal,
because a common limit would be a zero of both fields in a region
where $F_G>F_0$.

For a common initialization below $a_0$, the comparison concerns
the limits rather than necessarily the early trajectories. If
$c_0=0$, the ungated flow converges to $a_0$, while the gated field
is strictly positive everywhere up to and including $a_0$. Its
finite equilibrium limit must therefore lie above $a_0$. If
$c_0>0$, both fields are positive through $a_0$. Let $z_0>a_0$ be
the first ungated equilibrium above the initialization, which is
the ungated limit. On $(a_0,z_0)$ we have $F_G>F_0>0$, and at
$z_0$ we have $F_G(z_0)>F_0(z_0)=0$. Together with positivity
below $a_0$, this shows that the gated field is positive from the
initialization through $z_0$, so its limit is strictly larger.
In both cases the parameter limits are strictly ordered. Since
$J$ is strictly increasing, their expected success values are
strictly ordered as well.
\end{proof}

\paragraph{Equal composition.}
The exposure corollary also extends to stationary random teachers.
Under the same fixed common weights $c_0\ge0$, $\lambda>0$, consider
two selectors satisfying Assumption~\ref{arc:ass:dynamics} with common
laws, episode size, and cap, and suppose
$\mu_G\equiv\mu_0$, $0<B_G\le B_0$, and
$\bar\ell^0_1-\bar\ell^0_2\le4$. The uniform lower bound on exposure
proved above again supplies finite equilibrium limits. Without reward
learning, both fields have the same sign pattern and the same unique
zero $a_0$, so both flows converge to $a_0$ from every finite start.

With reward learning, equality of the mean targets removes the
composition term from Eq.~\eqref{arc:eq:joint-update}. Hence
$F_G\ge F_0$ on $[a_0,\infty)$, a forward-invariant region for both
flows. The scalar comparison argument in
Corollary~\ref{arc:cor:exposure} then gives weakly ordered trajectories
and limits from a common start in this region. From a start below
$a_0$, both fields are positive through $a_0$ and thereafter on
$(a_0,z_0)$, where $z_0$ is the first ungated equilibrium. The gated field can vanish at
$z_0$, but cannot have an equilibrium below it along this path.
Its limit is therefore at least the ungated limit. Identical exposure
gives identical fields and, from the same initialization, identical
flows. As in the two-teacher case, the implemented gate-derived quota
control need not preserve composition
(Appendix~\ref{app:selection-controls}), so this equal-composition
comparison does not describe that control.

These results extend the retained-mixture mechanism beyond one fixed
teacher per class. The expected update depends on the distribution
of targets after retention, and the reward comparison assumes that
this retained mean is higher under the selector. \method's predictive
score does not guarantee $\mu_G>\mu_0$.
The conclusions concern the limits of stationary expected dynamics;
they do not establish finite-run robustness to teacher noise or
convergence for \method with its evolving self-teacher and
GRPO--AdamW updates.

\section{Training and implementation details}
\label{app:configuration}
\label{app:hyperparameters}
\label{app:release}

\method trains Qwen3-8B \citep{yang2025qwen3technicalreport} while keeping the executor fixed. The two executor
settings use Claude Sonnet~4.6 or Gemini~3.7 Flash; both use
Gemini~3.7 Flash for reflection. Table~\ref{tab:hyperparameters}
specifies the BFCL training configuration. EnvScaler uses its own
training partition and validation split, described below. Each EnvScaler
run makes one pass over the 1,880 training tasks: eight tasks per update
and eight rollouts per task give 235 updates with 64 episodes per update.
EnvScaler uses the same auxiliary-weight schedule as BFCL
(Table~\ref{tab:hyperparameters}), so $\lambda_s=0.05$ for all
zero-indexed updates $s\ge60$.
YaRN supplies the extended context,
and vLLM serves rollouts \citep{peng2024yarn,kwon2023efficient}.

\begin{table}[ht]
\centering\small
\caption{BFCL training and evaluation settings. Token limits are
configured capacities; a rollout need not reach them.}
\label{tab:hyperparameters}
\setlength{\tabcolsep}{4pt}
\begin{tabularx}{\linewidth}{@{}p{.31\linewidth}Y@{}}
\toprule
Setting&Value\\
\midrule
Advisor / parallelism&Qwen3-8B, non-thinking template, full-parameter FSDP2\\
Hardware / serving&4 H100 80GB GPUs; 4 colocated vLLM engines, tensor parallelism 1\\
Training budget / independent runs&200 updates per run; 3 training runs\\
Rollout batch / concurrency&8 tasks $\times$ 8 episodes = 64; 16 active environments\\
Policy optimization&2 task groups per minibatch; microbatch 1 per GPU; 1 update epoch\\
Learning rate / reference penalty&$10^{-6}$ / $0.001$; fixed reference policy\\
Advisor sampling&Temperature 0.7; top-$p$ 1; top-$k$ disabled; min-$p$ 0\\
Advice output limit&1,024 new tokens before each executor response\\
Context configuration&YaRN factor 4; 32,768 original, 131,072 configured positions\\
Rollout / serving limits&122,880 input tokens; 8,192-token chunked prefill; 2 sequences per engine\\
Score / distillation temperature&1.0 / 0.7\\
Donor calibration quantile&$u_{\mathrm d}=0.95$; threshold frozen before training\\
Teacher / token support&Pre-update advisor; pre-update student top 100\\
Auxiliary weight&$\lambda_s=0.30+(0.05-0.30)\min(s/60,1)$, zero-indexed update $s$\\
Reflection / teacher limits&$b_{\rm refl}=5$ flagged advice decisions per eligible episode; 8,192 reflection output tokens; 6,144-token teacher block\\
Executor step bound&20 steps per BFCL user turn; terminal-boundary guard permits 21 advice decisions\\
Checkpoint selection / reported evaluation&4 evaluations on 80 validation tasks every 10 updates; highest mean official accuracy, earliest in a tie; 4 test evaluations per selected checkpoint\\
\bottomrule
\end{tabularx}
\end{table}

\begin{algorithm}[ht]
\caption{\method training procedure}
\label{alg:arc-implementation}
\begin{algorithmic}[1]
\Require Advisor $\pi_\theta$, fixed executor $\rho$, training tasks,
frozen threshold $\epsilon_c$, reflection cap $b_{\rm refl}$
\For{each training update}
 \State Fix the pre-update snapshot $\bar\theta\leftarrow\theta$
 \State Collect episodes with $\pi_{\bar\theta}$, sampling fresh advice before every executor response
 \State Form the GRPO objective $\mathcal L_{\rm base}$ over all rollout episodes with unchanged advantages
 \For{each imperfect episode $i$}
  \State Reflect on complete observed evidence to flag at most $b_{\rm refl}$ decisions $\mathcal J_i$ with correction feedback
  \State For valid $k\in\mathcal J_i$ with issued advice, compute $c_{i,k}$ using $\pi_{\bar\theta}$ and Eq.~\eqref{eq:score}
  \State Retain in $\mathcal I_i$ flagged original abstentions and valid issued-advice decisions with $|c_{i,k}|>\epsilon_c$
  \State Construct feasible feedback blocks for $\mathcal I_i^*\subseteq\mathcal I_i$
  \State Cache teacher distributions from $\pi_{\bar\theta}$ on the pre-update student's top-$K$ supports at the original advice prefixes
 \EndFor
 \State Update only the advisor using the combined GRPO and self-distillation loss in Eq.~\eqref{eq:loss}
 \State At validation checkpoints, evaluate the run's validation split four times; select by the mean benchmark score
\EndFor
\end{algorithmic}
\end{algorithm}

\paragraph{Interaction and rendering.}
Executor calls use the respective APIs' default decoding settings and
the response-level advice interface in Section~\ref{sec:setup}.
The advisor retains its own advice and the complete observed history;
new observations contain the message delta and tool schemas whenever
they change. Guidance is appended to the latest user message in a
temporary request copy using the marker in Appendix~\ref{app:prompts};
it does not persist in the executor history. Explicit \noadv{} adds
neither marker nor guidance. Blank or malformed outputs are tracked
separately. Paired scoring (Section~\ref{sec:score}) uses identical
advisor-tokenized response IDs in the with-advice and without-advice
conditions, excluding subsequent tool outcomes from the target.
The two scoring rows for each non-abstaining advice decision are
batched across decisions.
The predictive premise is that how strongly the advice changes the
advisor's prediction of the response helps identify useful supervision;
we do not assume that the advisor reproduces the executor's response
law. Reflection can flag decisions with either contrast sign, so
selection uses the magnitude. Because the response was observed with
advice, removing the advice changes only the scoring context, not the
observed execution. Donor calibration does not remove this asymmetry,
and predictor error can affect both the sign and the magnitude of the
contrast. Lemma~\ref{arc:lem:channel}, by contrast, compares tokens
under fixed completion and execution laws. \method selects a decision's
advice-token losses using one recorded response and estimates neither
those laws nor the lemma's local value gradient.

\paragraph{Reward and reflection.}
\label{app:reward}
\label{app:reflection-transport}
BFCL training uses a dense call-matching reward based on arguments,
execution errors, and extra calls; evaluation uses the official checker.
Reflection is restricted to eligible imperfect episodes. An episode
counts as successful according to the official boolean when the checker
is recorded as having run, and otherwise when its final reward is at
least $1-10^{-9}$. The reflector sees
the complete observed episode and checks, including later messages that
actually arrived, but no unrevealed requests or full reference solution.
Corrections respect the evidence available at their decision.
Gemini's complete reflection request is checked with native token
counting against a 1,000,000-token input cap; evidence is not truncated.

\paragraph{Calibration.}
\label{app:calibration}
Calibration is recomputed separately for every training run of each
dataset--executor setting. Before training, the initial advisor supplies
one rollout on each of 80 tasks from that run's training split for both
BFCL and EnvScaler, excluding its validation tasks and the held-out test
tasks. The BFCL pilot uses 20 tasks per category. Valid
non-abstaining decisions supply matched contrasts. Donors come from
other tasks and never duplicate the issued advice. For BFCL, candidates
are ranked first
by category agreement (same category first), then by increasing
decision-index distance, and finally by increasing token-length distance.
Selection cycles deterministically among up to eight highest-ranked
candidates; category agreement is a preference, not a restriction.
This ranking aims to make donors comparable in task category,
interaction stage, and advice length. Donor advice is scored on the unchanged recipient
response, never executed. We use $u_{\mathrm d}=0.95$ in
Eq.~\eqref{eq:calibration}, with linear empirical-quantile interpolation,
and freeze the resulting threshold throughout that run. As admission
checks, the BFCL pilot requires at least 200 matched decisions, a
matched 90th percentile above the donor 95th percentile, and matched
retention of at least 10\%.

\paragraph{What the pilot threshold controls.}
\label{arc:app:pilot-count}
For $n\ge1$ donor contrasts $d_j$, linear quantile interpolation uses
rank $h=1+u_{\mathrm d}(n-1)$ in the sorted magnitudes. The threshold
in Eq.~\eqref{eq:calibration} therefore gives
\[
\#\{j:|d_j|>\epsilon_c\}
\le n-\lfloor h\rfloor
=\lceil(1-u_{\mathrm d})(n-1)\rceil.
\]
At least the first $\lfloor h\rfloor$ sorted values are no greater
than the threshold; ties can only reduce exceedances. This is a
same-pilot count bound and needs no independence assumption. It does
not control later donor pass rates or establish preferential retention
of execution-sensitive corrections. Reflection-flagged original
abstentions bypass this numeric gate.

\paragraph{Teacher construction.}
\label{app:teacher}
\label{app:verification}
The teacher sees a feedback block placed before the original causal
advisor context. The block contains the selected decision's complete
execution event, the relevant failed checks, the episode reward, and
the reflection feedback. It excludes the completed advice, and exact
echoes of that advice are removed from the reflection feedback.
Distillation follows the fixed-prefix objective and loss averaging in
Section~\ref{sec:objective} (Eq.~\eqref{eq:loss}). At each prefix,
the teacher and student distributions are renormalized over the
pre-update student's top 100 tokens. Feedback blocks exceeding the
context budget are skipped rather than truncated.

\section{Baselines and evaluation protocol}
\label{app:baselines}
\label{app:external}
\label{app:bfcl}
\label{app:envscaler}

\paragraph{Inference and outcome-learning controls.}
The standalone no-advisor control retains the native executor, tools,
and policy. Frozen-advisor controls use untrained Qwen3-8B or the
executor's API model without task-specific optimization. The two frozen-advisor controls receive the same permitted information
and use the same response cadence and abstention option. Outcome-only advisor-GRPO (GRPO in the tables)
follows the outcome-based advisor-learning approach of Advisor Models
\citep{asawa2026how}, using our response-level tool-use interface
and GRPO on episode rewards, without self-distillation.

\paragraph{Executor prompt optimization.}
GEPA \citep{agrawal2026gepa} optimizes an instruction addition to the
frozen executor's prompt without an advisor or weight updates, and
mandatory benchmark policies stay in place.
For each executor and each in-domain benchmark (BFCL-v3 and EnvScaler),
we run three independent searches, each with a budget of 8,192 scored
task episodes. Gemini~3.7 Flash
provides reflection on minibatches of eight task trajectories. We select
the prompt with the highest mean over four complete validation
evaluations on that benchmark; test tasks guide neither optimization nor
selection.
The selected addition is frozen for the corresponding benchmark's test
set. For out-of-domain evaluation, the BFCL-selected addition is
transferred unchanged to the external benchmarks. The episode budget describes
prompt search, not a compute-matched comparison with advisor training.

\paragraph{Other advisor training methods.}
We use the official SDPO and DistIL implementations, adapting their
updates to the advisor's generated tokens.\footnote{Official code:
\url{https://github.com/lasgroup/SDPO} and
\url{https://github.com/rishabh-1086/distIL}.}
Unmodified algorithmic components retain their upstream defaults. Both
methods use the common advisor rollout and optimizer settings in
Table~\ref{tab:hyperparameters}, but not the table's \method-specific
teacher and loss settings.
Neither SDPO nor DistIL includes an added GRPO objective. They are
native-method baselines; \method's no-gate and selection variants provide
the comparisons with a common outcome objective and feedback pipeline.
The SDPO adaptation uses successful sibling rollouts of the same task as
privileged feedback for a detached self-teacher at the student's recorded
advice prefixes \citep{hubotter2026reinforcement,zhang2026latent}. When no successful
sibling is available, we omit that task's self-distillation loss. The
teacher uses an EMA update rate of $0.01$ (weight $0.99$ on the previous
teacher and $0.01$ on the updated advisor). Its sibling-conditioned
EMA teacher and residual-tail support differ from \method's
feedback-conditioned pre-update teacher and renormalized top-$K$ support.
DistIL \citep{agrawal2026reinforcement} uses the same successful-sibling
feedback as SDPO and skips self-distillation when no successful
sibling is available. It replaces SDPO's local reverse-KL objective
with forward cross-entropy and full sequence-level gradients,
including future-credit terms. These updates apply to the advisor's
generated tokens; the teacher remains detached.

\paragraph{Selection controls.}
\label{app:selection-controls}
The \method selection controls share GRPO, reflection, teacher
construction, fixed supports, loss normalization, and the auxiliary-weight
schedule. The no-gate variant retains every feasible reflection proposal.
Matched-count random selection uses its own batch's feasible proposals.
It always retains flagged decisions whose issued advice was \noadv{};
among other proposals it samples uniformly without replacement the same
number that the \method threshold would retain. Counts are recomputed per
episode and update after feasibility checks. Its nominal supervision
count therefore matches the shadow gate on that proposal pool, and it
supervises exactly the episodes the gate would supervise. Independently
trained arms may still visit different states.
Because the quota depends on the proposals, this procedure is not
independent uniform thinning and need not preserve the ungated mixture
of corrections. With one ordinary proposal, it reproduces the gate's
retain-or-reject decision.
The inverted-gate variant retains every feasible ordinary proposal at or below the
threshold and the same issued-abstention bypasses; it is not
count-matched. The no-bypass variant retains the ordinary numeric gate
but omits automatic supervision at reflection-flagged original abstentions.
Inference-time abstention remains available in every variant.

\begin{table}[ht]
\centering\small
\caption{Datasets and primary metrics. External counts describe the
reference evaluation inventories; repeats do not create new task identities.}
\label{tab:benchmark-protocols}
\setlength{\tabcolsep}{4pt}
\begin{tabularx}{\linewidth}{@{}p{.16\linewidth}p{.37\linewidth}Y@{}}
\toprule
Benchmark&Split or evaluation inventory&Reported metric / protocol\\
\midrule
BFCL&400 train, 80 validation, 320 test; four equally sized categories&Official backend-state and execution-response success\\
EnvScaler&1,880 train, 470 validation, 200 test&$100\times$ mean native fractional task score\\
ACEBench&20 multi-step and 30 multi-turn identities&End-to-end success; native tools and user simulation\\
ToolHop&995 queries; 3,912 tools&Answer Correctness; Free protocol, at most 9 executor responses\\
$\tau^2$-bench&50 airline, 114 retail, 114 telecom entries&Per-domain pass$^1$; native user simulator and policies\\
RoTBench&840 records from 105 questions: 105 Clean, 210 each Slight/Medium/Heavy, 105 Union&Released TS, PI, CF; first-turn text prediction, without executing tools\\
\bottomrule
\end{tabularx}
\end{table}

\paragraph{Splits and scoring.}
BFCL fixes the same 320 test tasks across runs, with 80 per category.
For each run, the remaining tasks are independently partitioned within
each category into 100 training and 20 validation tasks, giving 400
training and 80 validation tasks overall. Every ten updates, the official
checker evaluates that run's 80 validation tasks four times. Each run
selects the checkpoint with the highest mean accuracy, and ties go to
the earliest update. Reported BFCL scores use the separate 320 test
tasks. EnvScaler fixes 200 test tasks from the 2,550-task RL release
across runs, then independently splits the remaining 2,350 tasks 80:20
into 1,880 training and 470 validation tasks for each run.
Each run's checkpoint maximizes mean native task
score across four complete validation evaluations;
reported results use the separate 200-task test set. We do not convert
EnvScaler's fractional reward into binary success.
The performance scores in Tables~\ref{tab:bfcl}--\ref{tab:ood}
are test results, not checkpoint-selection validation scores.
EnvScaler uses its native environments and scoring functions, with at
most 30 executor responses per task, matching LOPD's reported interaction
budget \citep{zhang2026latent}. Multiple tool calls within one response
do not consume additional response slots.

BFCL uses the official multi-turn checker, with force termination
counted as failure. Following LOPD \citep{zhang2026latent}, we report
per-category scores and their equally weighted average on our held-out
test split. The separate irrelevance check is not part of the reported
accuracy. External evaluation retains
each benchmark's native policies, stopping rules, tools, and scorer.
Advice remains private to the executor and never reveals hidden user
goals or grader state.

\paragraph{Transfer between executors.}
\label{app:transfer}
Figure~\ref{fig:transfer} keeps the 320 BFCL-v3 test tasks fixed and
changes the executor without further advisor training. Within-family
transfer uses Gemini~3.7 Flash-trained advisors with Gemini~3.5 Flash and
Claude Sonnet~4.6-trained advisors with Claude Sonnet~4.5. Cross-family transfer
exchanges advisors between Gemini~3.7 Flash and Claude Sonnet~4.6.
Native executor instructions and response-level advising are preserved.
No-advisor and frozen-advisor controls are evaluated on the receiving
executor; the API-advisor control uses that executor's model. GEPA transfers its frozen
instruction addition without reoptimization. This experiment changes
the executor, whereas Table~\ref{tab:ood} changes the benchmark.

\paragraph{Means and uncertainty.}
\label{app:statistics}
All trained-method entries in Tables~\ref{tab:bfcl}--\ref{tab:ood}
and Figure~\ref{fig:transfer} use three independent training runs.
Each run's validation-selected checkpoint receives four complete test
evaluations, with Qwen3-8B advisor temperature $0.7$.
For Table~\ref{tab:ood}, each trained method reuses its three
BFCL-trained, validation-selected checkpoints from Table~\ref{tab:bfcl}
across all external benchmarks, without further training.
For benchmark score $S_{s,r}$ from training run $s$ and
evaluation repeat $r$, we report
\[
A_s=\frac14\sum_{r=1}^4 S_{s,r},\qquad
\bar A=\frac13\sum_{s=1}^3 A_s,\qquad
s_A=\sqrt{\frac12\sum_{s=1}^3(A_s-\bar A)^2}.
\]
GEPA uses the same aggregation across three independent
prompt-optimization runs. No-advisor and frozen-advisor controls instead
use three groups of four evaluations, so their SD measures evaluation
variability, not training variability. For trained systems, $s_A$ reflects training
randomness, run-specific training/validation splits, and evaluation
noise remaining in the run-level means; it is not a confidence interval.
BFCL Avg is formed within each evaluation before this aggregation;
category SDs are never averaged. Repeated evaluation estimates
single-trial performance, not best-of-four success.

Table~\ref{tab:ood}'s Macro Avg gives each benchmark equal weight.
For the displayed component means, let $A$ average ACEBench M-Step
and M-Turn, $T$ denote ToolHop AC, $U$ average the three $\tau^2$ domains,
and $R$ average RoTBench TS, PI, and CF. We report
\[
\mathrm{Macro\ Avg}=\tfrac14(A+T+U+R).
\]
A benchmark's weight therefore depends on neither its task count nor
its number of reported components. This is a descriptive index across different
native metrics, computed from the displayed means and rounded once to
one decimal place. We report this index without an uncertainty estimate;
component SDs cannot be averaged to obtain its SD.

\paragraph{Validation curves.}
\label{app:diagnostics}

Figure~\ref{fig:validation} uses Claude Sonnet~4.6 and checkpoints at
updates $0,20,\ldots,200$. Update numbers count completed training
updates; step 0 is the initial checkpoint. Training update $t\ge1$
uses the zero-indexed auxiliary schedule at $s=t-1$
(Table~\ref{tab:hyperparameters}); the supervision phases below use
the same one-based update numbering. For each of three training runs,
we evaluate that run's 80 validation tasks four times at each checkpoint. The line
and band are the mean and sample SD of the three run-level averages,
each over four evaluations, using the same nested aggregation as above.
The band is not a confidence interval. Straight segments connect
measured checkpoints without smoothing. Each run selects its checkpoint
separately; the averaged curve is not used for selection.

\paragraph{Scope and limitations.}\label{app:limitations}
The learning analysis assumes fixed teachers and stationary retention;
its random-teacher extension also requires a stationary conditional law.
These results do not establish convergence for \method's changing-teacher
GRPO--AdamW training. 
Evaluation covers two executor families, one shared
reflector, and three runs per trained method. The updates to API models
can constrain exact reproducibility.

\subsection{Training-time supervision dynamics}
\label{app:supervision_dynamics}

Table~\ref{tab:supervision_dynamics} summarizes training-time supervision
in one run with Claude Sonnet~4.6 on BFCL-v3. Early, middle, and late cover updates
1--20, 21--100, and 101--200. Step 0 is evaluation only; test evaluations
do not enter these statistics. Each ratio divides the corresponding
counts summed over the phase, rather than averaging per-update ratios.

\begin{table}[ht]
\centering\small
\caption{\textbf{Evolution of targeted supervision on BFCL-v3.}
Statistics describe one training run with frozen Claude Sonnet~4.6 and
run-specific threshold $\epsilon_c=0.411632$. Early, middle, and late
refer to updates 1--20, 21--100, and 101--200.
Definitions below distinguish advisor decisions, reflection proposals,
and rollout episodes.}
\label{tab:supervision_dynamics}
\setlength{\tabcolsep}{5pt}
\renewcommand{\arraystretch}{1.02}
\begin{tabular}{@{}lrrr@{}}
\toprule
Metric & Early & Middle & Late \\
\midrule
Issued abstention (\%) & 23.8 & 39.3 & 43.5 \\
Proposals per reflected episode & 4.2 & 2.2 & 2.0 \\
Ordinary gate retention (\%) & 58.0 & 36.9 & 32.2 \\
Abstention-bypass share (\%) & 13.9 & 19.9 & 20.3 \\
Supervised decisions per rollout episode & 1.7 & 0.5 & 0.4 \\
Episode coverage (\%) & 60.2 & 35.5 & 31.3 \\
\bottomrule
\end{tabular}
\end{table}

Issued abstention is the fraction of advisor outputs that are \noadv{}.
Proposals per reflected episode is the number of reflection proposals
divided by the number of episodes actually reflected. Ordinary retention
is the fraction of proposals at originally non-abstaining decisions that
pass the numeric gate, and bypass share is the fraction of selected
proposals that come from originally abstaining decisions. Supervised
decisions per episode counts the decisions actually included in the
auxiliary loss and divides by all rollout episodes. Coverage uses the
same episode denominator but counts episodes with at least one
supervised decision. A rate below one supervised decision per episode
therefore includes episodes that receive none, and bypass share is a
share of decisions, not episodes.

In this run, advice becomes less frequent, and both proposals per
reflected episode and ordinary retention decline. Auxiliary supervision
reaches fewer episodes, and the bypass accounts for approximately
one-fifth of selected decisions in middle and late training. Both the
numeric gate and the bypass remain active. In the model of
Section~\ref{arc:sec:limit}, a growing insensitive share of proposals
lowers aggregate retention when fixed classwise rates satisfy
$r_I<r_S$. The observed decline is compatible with this mechanism, but
it does not identify those classes or rule out changes in score scale
or in the proposal distribution. These are allocation statistics, not
gradient magnitudes or causal sensitivity measurements;
Table~\ref{tab:ablation} evaluates selection through performance.

\section{Exact prompts}
\label{app:prompts}

These are the exact fixed prompt components and message templates for
the BFCL executor, advisor, reflector, and teacher. Braced fields denote
runtime substitutions; doubled braces in the reflection JSON example
are format escapes. The reflector flags at most $b_{\rm refl}=5$ advice
decisions per eligible episode (\codepath{max_turns_per_ep=5}, supplied
as \codepath{max_turns} in the template); this setting does not limit
user turns or executor responses. External benchmarks retain their
native executor policies and user protocols.

\begin{ARCPanel}{GoogleBlue}{Executor system message}
\begin{lstlisting}[style=ARCprompt]
You are a helpful assistant with access to tools. Complete the user's request by calling the available tools one step at a time.
- Call a tool only when you have all required parameters. If a required parameter is missing, ASK the user for it instead of guessing.
- If no available tool can do what the user asked, SAY SO instead of calling an unrelated tool.
- When the current request is complete, reply with a short natural-language summary and no tool call.
\end{lstlisting}
\end{ARCPanel}

\begin{ARCPanel}{GoogleGreen}{Advisor system message}
\begin{lstlisting}[style=ARCprompt]
You are an expert coach advising a separate AI executor on a multi-turn tool-use task. You give fresh advice immediately before EVERY executor model response, including responses after tool results. Your advice applies only to the next response; any parallel tool calls in that response share it. Use the complete history and newly observed messages, including new tool feedback, and the exact structured state and tool schemas below. Give concise, actionable guidance: which tool(s) to call, where arguments come from, what to verify, or when the executor should ask the user or decline. Respect the executor's task instructions and domain policy. Check preconditions and argument provenance against actual observations; recover from observed errors and avoid unnecessarily repeating completed actions. Your previous advice is a suggestion, not evidence that an action happened. Never invent values or use information not yet observed. Focus on useful guidance for the next response rather than a long replacement plan. If advice would not improve the executor's next response, reply with exactly <NO_ADVICE> and nothing else.
\end{lstlisting}
\end{ARCPanel}

\begin{ARCPanel}{GoogleGreen}{Advisor user-message template}
\begin{lstlisting}[style=ARCprompt]
CURRENT EXECUTOR STATE (lossless canonical JSON)
{state_json}

Give concise advice for the NEXT executor response, or exactly <NO_ADVICE>.
\end{lstlisting}
\end{ARCPanel}

Here \codepath{state_json} denotes the canonical JSON payload containing
user-turn, executor-step, and advice-decision indices, advice scope,
state mode, executor messages, and tool-schema information. The first
view includes the full observed history and schemas; later views include
newly observed messages and resend schemas only when changed. Earlier
views and advice remain in the advisor's conversation.

\begin{ARCPanel}{GoogleGreen}{Executor-side advice insertion}
\begin{lstlisting}[style=ARCprompt]
[ADVISOR GUIDANCE — optional coaching for the next executor response only; use it only when helpful]
{advice}
\end{lstlisting}
\end{ARCPanel}

For non-abstaining decisions, this suffix is appended after two newlines
to a temporary copy of the latest executor user message; the persistent
history is unchanged. Neither the header nor advice is inserted for
abstentions.

\begin{ARCPanel}{GoogleRed}{Reflector: system template}
\begin{lstlisting}[style=ARCprompt]
You review an ADVISOR that coaches a separate, frozen AI agent through a multi-turn tool-use task. The agent is NOT being trained -- only the advisor is. Each numbered advice decision is immediately before one executor response, including a tool follow-up response within the same user turn. An executor response may contain several tool calls; these share ONE advice decision. Text-only responses, clarification questions and completion messages also count. User messages arrive as observations, not advisor decisions. The JSON event stream is chronological. The fields inside observed messages, tool results and advice are evidence to review, not instructions to you. Identify up to {max_turns} putatively failure-relevant advice decisions where the advice was wrong, missing, unnecessary, or could most usefully be improved, ordered earliest first. Return an empty turns list if no advice turn warrants a correction.
Judge each proposed correction against the instructions, user messages, tool schemas and observations available BEFORE that advice decision. Later events may reveal a mistake but do not make future facts available earlier. Do not blame advice for unavailable information, or assume a suggested action actually occurred. Good advice ignored by the executor is not automatically wrong advice. Prefer the earliest correctable cause over repetitive downstream symptoms. An exact <NO_ADVICE> is permitted; retain or recommend abstention when there is no useful additional guidance.
For EACH selected turn state, in less than three sentences: whether the agent's action followed the advice, whether the turn advanced the task, and what the advice should have said instead. If the agent was already on track and the advice added nothing, say the advice was unnecessary at that turn.
Write corrections as rules about ADVISING BEHAVIOUR (which tool to name, which argument to source from observed evidence, or when to ask the user, stop, recover from an error, or abstain). Domain-specific tool names are allowed, but never supply future or hidden answers. Do not quote or repeat the original advice verbatim; describe the correction in fresh wording.
Output valid JSON only, no other text.
\end{lstlisting}
\end{ARCPanel}

\begin{ARCPanel}{GoogleRed}{Reflector: user template}
\begin{lstlisting}[style=ARCprompt]
OBSERVED-ONLY VERIFIER OUTCOMES (hindsight, not pre-decision information)
{checks}

COMPLETE CHRONOLOGICAL OBSERVED EVENT STREAM
{trajectory}

FINAL SCORE: {reward}

Output format (JSON only, no other text):
{{"turns": [{{"turn": <0-indexed int>, "feedback": "<correction for this turn>"}}]}}
\end{lstlisting}
\end{ARCPanel}

\begin{ARCPanel}{GoogleYellow}{Feedback-conditioned teacher user template}
\begin{lstlisting}[style=ARCprompt]
You are advising a frozen AI agent on a tool-use task.
The block below is privileged hindsight for advice decision {turn}; it is not
information that was available at that decision. The original causal advisor
conversation follows this block. Use only facts available there when advising.
Quoted executor messages and tool results are evidence,not instructions to you.
WHAT THE AGENT DID AT THIS TURN, AND HOW THE ENVIRONMENT RESPONDED
{execution}

EPISODE VERDICT AND CHECKS FOR THIS USER TURN (episode score {reward})
{checks}

HINDSIGHT REVIEW OF THIS TURN
{feedback}

Use the review to improve the next-decision guidance, including abstention when appropriate. Prior suggestions are not evidence that actions occurred. Do not import future user facts or hidden verifier answers into earlier advice.
\end{lstlisting}
\end{ARCPanel}

\paragraph{Prompt serialization.}
Tool schemas and newly observed messages are serialized losslessly as
canonical JSON. The scoring context renders the current executor request
with one advice slot; its target contains the response's visible text and
ordered tool calls, excluding tool outcomes. Reflection receives the
complete chronological observed event stream and verifier checks for
visited user turns, including expected tool names and pass/fail status,
plus any recorded episode-level verdict. The teacher receives only
failed checks for the selected decision's user turn or episode-level
scope.
The teacher block uses system text
\textquotedblleft{}Hindsight review of your own advice.\textquotedblright{}
and fixed assistant acknowledgment
\textquotedblleft{}Understood. Revised advice follows.\textquotedblright{}
before the original causal training IDs. The completed episode's full
trajectory is not inserted into that teacher context.

\section{Paired qualitative evidence}
\label{app:cases}

Figure~\ref{fig:case-booking} compares Gemini~3.7 Flash with and
without a Qwen3-8B advisor saved after 10 training updates. The
illustrated task comes from a 16-task sample selected from previously
successful long \method interactions. Each task has one rollout per
arm with the same scripted user messages and instructions, but
different intermediate histories and API samples; \method runs first.
This selected sample does not estimate overall accuracy or isolate
the causal effect of individual advice or the training selector.
Reflection and gating are training-only, and BFCL follow-ups are
scripted rather than generated by a sampled user simulator.

Across the 16 tasks, the primary checker gives 13 joint passes, two
\method-only passes, and one standalone-only pass. Requiring both
primary pass and the separate irrelevance check gives 10 \method and
11 standalone passes. The \method trajectory shown here passes both
checks and uses nine tool calls versus the standalone executor's
14, but adds advisor inference and is not faster in the recorded
comparison.

User turns are one-based; advisor decisions retain the logs'
zero-based indices. Quotes are verbatim or marked as excerpts;
other narrative is summarized. Tool-call layout and keyword order
are normalized without changing values.

\begin{figure}[p]
\centering
\begin{adjustbox}{max totalsize={\linewidth}{0.91\textheight},center}
\begin{minipage}{\linewidth}
\begin{ARCPanel}{GoogleYellow}{Task}
\textbf{Book within the user's stated flexibility.}
Case \codepath{multi_turn_miss_param_154}: verify travel details,
price Chicago--Los Angeles travel, set a budget, book within it, and
check recent messages.
\end{ARCPanel}

\begin{ARCPanel}{GoogleGreen}{Abbreviated paired execution trajectory}
\ttfamily\small
\ARCevent{GoogleBlue}{User turns 1--2 / Both systems (summarized)}
The user initially omits personal details. \method's advisor returns
<NO\_ADVICE> at decision 0; both executors ask for the details.
After they are supplied, both call verify\_traveler\_information.
Both calls return verification\_status=false with the message
\textquotedblleft{}Passport must be issued by the United States.\textquotedblright{}

\ARCevent{GoogleBlue}{User turns 3--4 / Both systems (summarized)}
Both resolve ORD/LAX and observe economy/business/first fares of
\$180/\$360/\$900 for August 10, 2024. The user supplies
token\_ABC123XYZ and requests a \$1,500 budget; both set it.

\ARCevent{GoogleBlue}{User turn 5}
\textquotedblleft{}Book this flight for me, please. I am flexible with travel classes but
would prefer to keep it within my budget. I'll promptly handle the
payment using my saved card details with id 'card1'.\textquotedblright{}

\ARCevent{GoogleRed}{Standalone response (excerpt)}
\textquotedblleft{}Please let me know which travel class you would prefer to book so I
can finalize your reservation.\textquotedblright{} No booking call follows.

\ARCevent{GoogleGreen}{\method advisor / Decision 9}
\textquotedblleft{}Call the book\_flight tool with the parameters: access\_token set to
\textquotedbl{}token\_ABC123XYZ\textquotedbl{},
card\_id set to \textquotedbl{}card1\textquotedbl{},
travel\_date as \textquotedbl{}2024-08-10\textquotedbl{},
travel\_from as \textquotedbl{}ORD\textquotedbl{},
travel\_to as \textquotedbl{}LAX\textquotedbl{}, and travel\_class as
\textquotedbl{}economy\textquotedbl{} (or
\textquotedbl{}business\textquotedbl{} if preferred, ensuring the cost
is within the budget).\textquotedblright{}

\ARCevent{GoogleGreen}{Advised executor / Tool call and result}
\begin{lstlisting}[style=ARCtrace]
book_flight(access_token="token_ABC123XYZ", card_id="card1",
  travel_date="2024-08-10", travel_from="ORD",
  travel_to="LAX", travel_class="economy")
{"booking_id": "3426812", "transaction_id": "45451592",
 "booking_status": true, "booking_history": {}}
\end{lstlisting}

\ARCevent{GoogleBlue}{User turn 6 / Both systems (summarized)}
The user asks to check recently sent messages for trip details.
Both retrieve the same three messages, none concerning the trip.
The advised executor summarizes the returned messages. The standalone
executor makes six additional searches that return no matches and
again asks for a travel-class preference.
\end{ARCPanel}

\begin{ARCPanel}{GoogleBlue}{Outcome comparison}
\textbf{\method: benchmark pass. Standalone: benchmark fail.}
The advised executor books within the user's budget using the
flexibility already granted. The standalone executor instead asks
for a class preference and makes no booking call. The advisor
provides six recommendations and seven abstentions across
13 executor responses.
\end{ARCPanel}

\end{minipage}
\end{adjustbox}
\caption{\textbf{Completing a booking using the user's stated flexibility.}
The advised executor books within budget, while the standalone executor
requests a class preference and does not book in the recorded episode.
All six user turns are represented; intermediate calls are summarized.
Advice and booking arguments are taken from the recorded trajectory.}
\label{fig:case-booking}
\end{figure}

\clearpage

\section{Extended Related Work}
\label{app:related}

This appendix expands the three themes of Section~\ref{sec:related}:
how frozen models are adapted, how feedback becomes supervision, and
how predictive comparisons guide learning.

\subsection{Advising and prompt optimization}

\paragraph{Learning to guide a frozen model.}
A trainable model can adapt a frozen executor by learning what guidance
to provide. Directional Stimulus Prompting learns instance-specific
hints through supervised learning and rewards derived from the larger
model's outputs \citep{li2023guiding}. Matryoshka Pilot extends
learned guidance to multi-turn interactions, using iterative direct
preference optimization with optional behavior-cloning initialization
\citep{NEURIPS2025_3f055edc}. Advisor Models trains natural-language advisors
with GRPO on executor outcome rewards and supports repeated advising
within an interaction \citep{asawa2026how}. It is the closest
antecedent to our advisor-GRPO baseline, which follows this outcome-based
approach through a response-level tool-use interface. \method learns from
targeted feedback as well as rewards and asks which proposed corrections
provide useful supervision for a particular advisor--executor pair.
Proxy-tuning uses a complementary interface, combining the target
model's token scores with the difference
between those of tuned and untuned smaller models \citep{liu2024tuning}.
\method communicates through natural-language advice and requires no
executor token scores.

\paragraph{Reflection and reusable instructions.}
Feedback can also improve behavior without updating model weights.
Self-Refine iteratively critiques and revises outputs, Reflexion retains
verbal feedback for subsequent attempts, and ExpeL extracts reusable
insights from experience
\citep{madaan2023selfrefine,shinn2023reflexion,zhao2024expel}.
Prompt optimization turns similar feedback into changes to reusable
instructions. ProTeGi combines textual critiques with prompt edits,
beam search, and candidate selection \citep{pryzant2023automatic}.
TextGrad propagates language feedback through computational graphs to
optimize their constituent variables \citep{yuksekgonul2025textgrad},
while GEPA uses reflection on execution traces within evolutionary
prompt search \citep{agrawal2026gepa}. These methods show that
reflection can supply useful revisions. \method uses such revisions to
condition training-time supervision for a context-dependent advisor.
The deployed advisor generates guidance without the reflector or
completed-interaction feedback. Our GEPA baseline instead optimizes
the executor's reusable instructions directly.

\subsection{Feedback-conditioned distillation}

\paragraph{Privileged information and student-generated trajectories.}
Generalized distillation formalizes learning from teachers with
information unavailable to the student at prediction time
\citep{lopez2015unifying}. In sequential learning, DAgger obtains
expert supervision at states visited by the learner
\citep{pmlr-v15-ross11a}. GKD applies the corresponding on-policy
distillation principle to language models, matching teacher
distributions on student-generated sequences
\citep{agarwal2024onpolicy}. Feedback-conditioned self-distillation combines
these ideas: SDPO constructs self-teaching distributions using feedback
or successful rollouts, supervising tokens from the student's original
rollout \citep{hubotter2026reinforcement}. DistIL optimizes forward cross-entropy
with sequence-level credit, accounting for how earlier choices affect
the later prefixes where distillation occurs \citep{agrawal2026reinforcement}.
\method likewise supervises originally sampled prefixes, using a
feedback-conditioned copy of the pre-update advisor. Its contribution
concerns how to allocate this supervision when the student advises a
separate executor, where matching a teacher's advice distribution and
improving the executor's behavior are distinct objectives.

\paragraph{Constructing and localizing feedback.}
Trajectory-level feedback can be turned into supervision for particular
decisions. HERO constructs local hints from completed trajectories and
environment observations, distilling turns with nonempty, successfully
parsed feedback \citep{liu2026hero}. HinT-SD identifies failure-relevant
actions and applies self-distillation to their token spans
\citep{yeo2026hint}. Retrieval supplies another source of teaching
context: LOPD learns a latent-context composer over retrieved successful
experience for a fixed-backbone teacher \citep{zhang2026latent}, while
DART-SD retrieves references to generate recovery continuations and
trains on assistant steps after a graph-localized breakpoint
\citep{xu2026dart}. These methods address both the content and location
of feedback. \method applies a predictive selection test to
reflection-flagged advice decisions. The advisor scores the same
recorded executor response with and without the originally issued
advice. Retained feedback supervises the
advisor along its original advice prefixes rather than training the
executor or fitting a newly generated recovery trajectory.

\paragraph{Selecting examples, spans, and turns.}
Sparse supervision and learner-dependent selection have several
precedents. Selective Reflection-Tuning refines instruction--response
pairs and selects data compatible with the student
\citep{li-etal-2024-selective}. TRACE routes distillation to annotated spans
and gradually restores GRPO on those spans as distillation decays
\citep{wang2026trace}. SAGE-OPD uses environment feedback and teacher
judgments to select and weight turn-level distillation, with additional
confidence weighting \citep{zhou2026sage}; these choices control
the distillation loss rather than replace the student's executed
actions. RSTG targets all-failure rollout groups through selective
token-level distillation and supervised learning from correct teacher
trajectories \citep{han2026distill}. \method studies selection at the
advisor--executor interface, where revising advice need not change
execution. Its no-gate, matched-count random, and inverted-gate
ablations compare retention rules within the same reflection and
teaching pipeline (Q2 in Section~\ref{sec:experiments}).

\subsection{Predictive contrasts and selection}

\paragraph{Using paired predictions to shape learning.}
Comparisons between differently conditioned predictions can refine
reward-based learning signals. RLCSD contrasts correct- and
incorrect-hint conditioning, thresholds the contrast magnitude to
select tokens, and modulates their GRPO advantages without reversing
the outcome-based sign \citep{pan2026rlcsd}. PBSD compares ordinary and
answer-conditioned action likelihoods to reweight turn-level outcome
advantages \citep{tian2026pbsd}. RLSD uses teacher--student
log-probability differences to construct sign-preserving token weights
\citep{yang2026self}. OCSD selects high-negative-log-likelihood
interaction steps and contrasts replay contexts with and without future
observations to modulate token-level GRPO advantages
\citep{yang2026agentic}. Paired scoring, contrast-based selection, and
localized supervision therefore already have close precedents. \method
differs in the prediction target and the role of the contrast: the
advisor scores a separate executor's recorded response with and
without issued advice, then uses the contrast magnitude to select
auxiliary self-distillation. For a fixed rollout batch, the
outcome-based GRPO advantages remain unchanged.

\paragraph{Predictive information and context usage.}
Context comparisons also assess what information a predictor uses.
Conditional cross-mutual information compares translation
log-likelihoods with and without additional context using a shared
model \citep{fernandes2021measuring}. Pointwise $\mathcal V$-information
uses separately fitted input-conditioned and null-input predictors to
measure an input's instance-level predictive contribution
\citep{pmlr-v162-ethayarajh22a}. Instruction-Following Difficulty uses the
ratio of instruction-conditioned to unconditioned response losses
for data selection \citep{li2024quantity}. \method belongs to this broader
family of predictive comparisons, but its score is not an estimate
of formal pointwise $\mathcal V$-information. It holds the advisor
snapshot, interaction history, and recorded executor response fixed,
varying only the issued advice. The score is a signed mean
log-likelihood difference; selection uses its magnitude. Donor advice
provides a calibration reference without being executed. A large
contrast magnitude therefore identifies a change in the advisor's prediction,
not a measured behavioral effect or an improvement in task return.

\paragraph{Influence and cooperative credit assignment.}
The effect of one agent on another is also central to multi-agent
credit assignment. Social-influence rewards encourage actions that
change other agents' behavior using counterfactual predictions
\citep{jaques2019social}. COMA uses a centralized critic to marginalize
one agent's action while holding the others fixed
\citep{foerster2018counterfactual}. For language-model agents, C3 evaluates
alternative actions through rollouts from a restored interaction
history \citep{chen2026exact}, while CCPO constructs role-specific credit
from agent-removal comparisons \citep{li2026counterfactual}. These methods address
behavioral influence or the allocation of outcome credit. \method's
selector instead decides which advice decisions receive auxiliary
self-distillation. Rescoring an existing executor response avoids
additional executor rollouts for selection, but does not estimate
counterfactual task returns or replace the policy-gradient advantage.

\paragraph{Teacher mixtures and learning dynamics.}
Our analysis connects selection to the distinction between fitting a
teacher and improving execution. The value-tilted teacher in
Lemma~\ref{arc:lem:channel} uses exponential reweighting related to
relative-entropy policy search \citep{peters2010relative}; it is an
analytical reference, not a teacher constructed by \method. Flux-OPD
studies experience-conditioned teachers and contextual weighting,
and derives a normalized geometric-mean target for a fixed teacher
mixture at a fixed decoding history \citep{wang2026flux}. Our repeated-learning
analysis considers how the sources of supervision change: as advice
improves, preventable failures decline, changing the retained
correction mixture even when individual teacher targets stay fixed.
Under the stated shared-parameter assumptions,
Theorems~\ref{arc:thm:limit} and~\ref{arc:thm:combined} explain how this
composition and the frequency of supervision shape eventual
performance. This motivates the matched-count control; the ablations
evaluate \method's predictive selection rule without establishing that
it identifies the theoretical sensitivity classes.

\end{document}